%% file: arxiv-main.tex
\documentclass[11pt,twoside]{article}
\usepackage{fullpage}
\usepackage{epsf}
\usepackage{fancyhdr}
\usepackage{graphics}
\usepackage{graphicx}
\usepackage{psfrag}
\usepackage{microtype}
\usepackage{subfigure}
\usepackage{algorithmic}
\usepackage{color,xcolor}
\usepackage[linesnumbered,ruled]{algorithm2e}
\DontPrintSemicolon
\usepackage{color}
\usepackage{tabularx}
\usepackage{amsthm}
\usepackage{amsfonts}
\usepackage{amsmath}
\usepackage{amssymb,bbm}
\usepackage{stackengine}
\usepackage{footnote}
\makesavenoteenv{tabular}
\makesavenoteenv{table}

\usepackage{thmtools, thm-restate}
\newcommand{\logsobo}{\lambda_*}
\newcommand{\oracle}{\mathcal{O}}
\newcommand{\transition}{\mathcal{T}}
\newcommand{\proposal}{\mathcal{P}}

\newcommand{\lipschitzness}{M_\dims}
\newcommand{\stepsize}{\eta}
\newcommand{\dissipative}{\mu}
\newcommand{\distantdissipative}{\beta}
\newcommand{\referencepoint}{x_0}
\newcommand{\warmstart}{M_0}

\newcommand{\Tmix}{\ensuremath{T_{\mbox{\tiny{mix}}}}}
\newcommand{\DTV}{\ensuremath{d_{\mbox{\tiny{TV}}}}}

\input{final_macros.tex}

\usepackage[square,sort,comma,numbers]{natbib}
\usepackage{url}
\usepackage[colorlinks,linkcolor=magenta,citecolor=blue, pagebackref=true,backref=true]{hyperref}
\renewcommand*{\backrefalt}[4]{%
    \ifcase #1 \footnotesize{(Not cited.)}%
    \or        \footnotesize{(Cited on page~#2.)}%
    \else      \footnotesize{(Cited on pages~#2.)}%
    \fi}

\usepackage{nicefrac}
\usepackage{comment}

\usepackage{chngpage}

 \usepackage{tabularx}%

\usepackage{enumitem}
\usepackage{booktabs}
\usepackage{caption}

\usepackage{bm,bbm}
\usepackage{mathtools}

\newtheorem{assumption}{Assumption}
\newcommand{\xiter}[1]{\ensuremath{X^{#1}}}

\makeatletter
\long\def\@makecaption#1#2{
        \vskip 0.8ex
        \setbox\@tempboxa\hbox{\small {\bf #1:} #2}
        \parindent 1.5em  
        \dimen0=\hsize
        \advance\dimen0 by -3em
        \ifdim \wd\@tempboxa >\dimen0
                \hbox to \hsize{
                        \parindent 0em
                        \hfil 
                        \parbox{\dimen0}{\def\baselinestretch{0.96}\small
                                {\bf #1.} #2
                                } 
                        \hfil}
        \else \hbox to \hsize{\hfil \box\@tempboxa \hfil}
        \fi
        }
\makeatother

\begin{document}

\begin{center}
{\bf{\LARGE{An Efficient Sampling Algorithm for Non-smooth Composite Potentials}}}

\vspace*{.2in}
 \large{
 \begin{tabular}{cccc}
  Wenlong Mou$^{ \diamond}$ & Nicolas Flammarion$^{ \diamond}$ &
  Martin J. Wainwright$^{\dagger, \diamond, \ddagger}$ &Peter L.
  Bartlett$^{\diamond, \dagger}$
 \end{tabular}

}

\vspace*{.2in}

 \begin{tabular}{c}
 Department of Electrical Engineering and Computer Sciences$^\diamond$\\
 Department of Statistics$^\dagger$ \\
 UC Berkeley\\
 \end{tabular}

 \vspace*{.1in}
 \begin{tabular}{c}
 Voleon Group$^\ddagger$
 \end{tabular}

\vspace*{.2in}

\today

\vspace*{.2in}

\begin{abstract}
  We consider the problem of sampling from a density of the form
  \mbox{$p(x) \propto \exp(-f(x)- g(x))$,} where \mbox{$f:
    \mathbb{R}^d \rightarrow \mathbb{R}$} is a smooth and strongly
  convex function and \mbox{$g: \mathbb{R}^d \rightarrow \mathbb{R}$}
  is a convex and Lipschitz function. We propose a new algorithm based
  on the Metropolis-Hastings framework, and prove that it mixes to
  within TV distance $\varepsilon$ of the target density in at most
  $O(d \log (d/\varepsilon))$ iterations. This guarantee extends
  previous results on sampling from distributions with smooth log
  densities ($g = 0$) to the more general composite non-smooth case,
  with the same mixing time up to a multiple of the condition number.
  Our method is based on a novel proximal-based proposal distribution
  that can be efficiently computed for a large class of non-smooth
  functions $g$.
\end{abstract}
\end{center}


\section{Introduction}
Drawing samples from a distribution is a fundamental problem in
machine learning, theoretical computer science and statistics. With
the rapid growth of modern big data analysis, sampling algorithms are
playing an increasingly important role in many aspects of machine
learning, including Bayesian analysis, graphical modeling,
privacy-constrained statistics, and reinforcement learning.  For
high-dimensional problems, a standard approach is based on Markov
Chain Monte Carlo (MCMC) algorithms.  Such MCMC algorithms have been
applied to many problems, including collaborative filtering and matrix
completion~\citep{salakhutdinov2007bayesian,salakhutdinov2008bayesian},
large-scale Bayesian learning~\citep{welling2011bayesian}, text
categorization~\citep{genkin2007large}, graphical model
learning~\citep{besag1993spatial,wainwright2008graphical}, and
Bayesian variable selection~\citep{yang2016computational}.  Moreover,
sampling algorithms have also been used for exploration in
reinforcement learning~\citep{ghavamzadeh2015bayesian} and
privacy-preserving machine learning~\citep{dwork2014algorithmic}.

For statistical $M$-estimation, various non-smooth regularization
functions---among them the $\ell_1$-norm and variants thereof---are
the workhorse in this field, and have been used successfully for
decades. The non-smooth nature of the penalty term changes the
statistical complexity of the problem, making it possible to obtain
consistent estimators for high-dimensional
problems~\citep{buehlmann20111statistics,hastie2015statistical}. In
the Bayesian setup, non-smooth priors are also playing a key role in
high-dimensional
models~\citep{seeger2008bayesian,ohara2009review,polson2011shrink}. The
Bayesian analogue of the $\ell_1$-penalty is the Laplacian prior, and
has been the subject of considerable
research~\citep{park2008bayesian,carvalho2008high,dalalyan2018exponentially}.

In this paper, we study the problem of sampling from a distribution
$\Pi$ defined by a \mbox{density $\pi$} (with respect to Lebesgue
measure) that takes the composite form
\begin{align}
\label{eq:target-sampling}  
   \pi(x) \propto \exp(- U(x) ),\quad \mbox{where $U = f + g$,}
\end{align}
where only the function $f$ need be smooth. In Bayesian analysis, the
density typically corresponds to a posterior distribution, where
$e^{-f}$ is the likelihood defined by the observed data, and $e^{- g}$
is a prior distribution. The function $f$ is usually accessed through
an oracle that returns the function value and the gradient evaluated
at any query point, while the function $g$ is explicitly known in
closed form and often possesses some specific structure. This
composite model covers many problems of practical interest in
high-dimensional machine learning and signal processing~\citep[see,
  e.g.,][]{RisGra14}.

In convex optimization, it has been shown that composite objectives of
the form $U = f + g$ can be minimized using algorithms that converge
as quickly as those applicable to smooth minimization
problems~\citep{BecTeb09}; in particular, these algorithms require a
gradient oracle for $f$ and a proximity oracle defined by the function
$g$.  However, the current state-of-the-art rate for the sampling
problem~\eqref{eq:target-sampling} still does not match its smooth
counterpart.  Specifically, to make the mixing time scale linearly
with dimension, the best known dependency~\citep{durmus2019analysis}
on the accuracy $\varepsilon \in (0,1)$ scales as $O ( d /
\varepsilon^2)$.  Their algorithm suffers from significant bias of the
unadjusted Markov chain, meaning that one has to make the step size
very small.  In addition, exponentially fast convergence rates have
not been achieved with non-asymptotic guarantees (see
Section~\ref{sec:related-works} below).

In this work, we close the gap between composite sampling problems and
smooth problems, by developing a Metropolis-adjusted algorithm with a
new proposal distribution inspired by the proximity operator. For
sampling from the distribution~\eqref{eq:target-sampling}, our
algorithm has mixing time scaling as $O(d \log
\tfrac{d}{\varepsilon})$ whenever the density $\pi$ satisfies a
log-Sobolev inequality and the function $g$ is convex and
$O(\sqrt{d})$-Lipschitz. Our results apply to a broad class of
problems where the proximal version of the sampling oracle associated
with the penalty $g$ is available, including the case of the Laplacian
prior. These guarantees improve upon existing algorithmic results for
sampling problem~\eqref{eq:target-sampling}, both in the dependence on
both $d$ and $\varepsilon$, and match the corresponding rate for the
Metropolis-adjusted Langevin algorithm in the smooth
case~\citep{dwivedi2018log}, up to a multiple of the condition number.

\begin{comment}
\begin{itemize}
    \item Lot of recent works in sampling.
    \item But they rely crucially on the log-smoothness of the target
      distribution.
    \item The same problem has been encountered by the optimization
      community and solved thanks to proximal methods.
    \item Some attempts to directly extend these algorithms to
      sampling problems are inconclusive.
    \item We keep the essence of the proximal method from optimization
      and adapt it to the sampling problem. We propose the first
      theoretically efficient algorithm for sampling from a non-smooth
      problem which obtains the same mixing time as its smooth
      counterparts.
\end{itemize}
\end{comment}


\subsection{Related work}
\label{sec:related-works}

Both MCMC algorithms and proximal point methods have been intensively
studied in different settings, and here we review the existing
literature most relevant to our paper.
 \vspace{-0,1cm} 
\paragraph{Metropolis-Hastings sampling:}
The Metropolis-Hastings algorithm dates back to seminal
work~\citep{metropolis1953equation,hastings1970monte,gelfand1990sampling}
from the 1950s.  This simple and elegant idea allows one to
automatically build a Markov chain whose stationary distribution is
the desired target distribution.  All Metropolis-Hastings algorithms
are based on an underlying proposal distribution, with the simplest
one being associated with a random walk.  Earlier work focuses on
asymptotic theory, including guarantees of geometric ergodicity and
central limit theorems for random-walk-based Metropolis proven under
various
assumptions~\citep{meyn1994computable,mengersen1996rates,roberts1996geometric,jarner2000geometric,roberts2001optimal,
  roberts2004general}.  Various coupling-based methods have proven
useful for proving non-asymptotic bounds, including coupling with
metric estimates, and conductance analysis. The former can be used to
prove convergence in Wasserstein metrics, whereas the conductance
approach leads to convergence guarantees in the total variation (TV)
distance.  The mixing rate of a Markov chain is intimately related to
its conductance~\cite{jerrum1988conductance,lovasz1999faster}, a
quantity that can be further related to the isoperimetric properties
of the target distribution~\citep{lovasz2007geometry}.


\paragraph{Langevin-based sampling:}

Many sampling algorithms for smooth potentials rest on the Langevin
diffusion, defined via the Ito stochastic differential equation
\begin{align}
\label{eq:langevin}
dX_t = -\nabla f(X_t) dt + \sqrt{2} dB_t,
\end{align}
where $B_t$ is a standard $d$-dimensional Brownian motion. Indeed,
when the function $f$ is smooth, the Langevin
process~\eqref{eq:langevin} has a stationary distribution with density
proportional to $\exp(-f)$; under mild conditions, the diffusion
process~\eqref{eq:langevin} converges to this stationary distribution
as $t \to \infty$. This perspective encompasses algorithms based on
simple discretization of the Langevin diffusion such as the unadjusted
Langevin algorithm
(ULA)~\citep{roberts1996exponential,dalalyan2017theoretical} and
variants proposed to refine the dependence of the mixing time on
different problem
parameters~\citep{cheng2017convergence,lee2018algorithmic,mangoubi2018dimensionally}.
Applying a Metropolis-Hastings step to the discretized Langevin
diffusion results in the Metropolis-adjusted Langevin algorithm
(MALA)~\citep{roberts1996exponential,
  bou2013non,eberle2014error}. Both ULA and MALA have been
well-understood when applied to smooth and strongly log-concave
potentials, with mixing rates $O(d/\varepsilon)$~\citep{durmus2017non}
and $O(d\log(1/\varepsilon))$~\citep{dwivedi2018log, chen2019hmc},
respectively.

When the potential is non-smooth, the drift of the Langevin SDE
becomes discontinuous, making the diffusion notoriously difficult to
discretize. Some past work has exploited smoothing techniques from
optimization theory to tackle this challenge, as we now discuss.

 \vspace{-0,05cm} 
\paragraph{Proximal algorithms:}

The Moreau-Yosida envelope~\citep{Mor62} of a function $g$ at scale
$\stepsize > 0$ is given by $ g^\stepsize(x) \mydefn \min_{y \in
  \real^d} \left \{ \frac{1}{2 \stepsize} \vecnorm{y-x}{2}^2 + g(y)
\right \}$.  Note that $g^\stepsize$ is a smooth approximation of the
function $g$, and the minimizing argument defines the \emph{proximity
  operator}
\begin{align}
  \label{EqnProximity}
\prox{\stepsize,g}{x} = \arg \min_{y \in \real^d} \left \{
\frac{1}{2\stepsize }\vecnorm{y-x}{2}^2+g(y) \right\}.
\end{align}
Note that the gradient of $g^\stepsize$ is related to the proximity
operator via the equality relation \mbox{$\stepsize \nabla g^\stepsize
  = \text{Prox}_{\stepsize,g}-Id$.}  Consequently, it is possible to
optimize non-smooth functions as efficiently as smooth ones if we are
given access to their proximity operator. This idea underlies a great
deal of recent progress in optimization, which step back from
black-box approaches and instead leverage the special structure of the
problem under consideration.  One striking example is the minimization
of functions of the composite form $f + g$, where both functions $f$
and $g$ are convex but only the function $f$ is
smooth. Proximal-gradient methods are based on the update $x_{n+1} =
\prox{\stepsize, g }{x_n-\stepsize \nabla f(x_n)}$, and are specially
appealing for solving problems where $g$ is
non-smooth~\citep{BecTeb09,WriNowFig09,ComPes11}.  Indeed, their
convergence rates match those obtained by gradient methods on smooth
problems; these fast rates should be contrasted with the slowness of
subgradient methods.  However, the efficiency of a proximal-gradient
method is predicated upon an efficient method for computing the
proximity operator.  Fortunately, many choices of $g$ encountered in
machine learning and signal processing lead to simple proximity
operators.


\paragraph{Non-smooth sampling:}

Pereyra et al.~\citep{pereyra2016proximal} proposed to sample from a
non-smooth potential $g$ by applying both the Metropolis-adjusted
Langevin and the unadjusted Langevin algorithms to its Moreau-Yosida
envelope.  Bernton~\citep{bernton2018langevin} analyzed the latter
algorithm in a particular case. Durmus et
al.~\citep{durmus2018efficient} extended these approaches to composite
potentials of the form $f+g$ by considering a smooth approximation of
the form $f + g^\lambda$, where $g^\lambda$ is a smooth version of the
non-smooth $g$, with the amount of smoothness parameterized by a
positive scalar $\lambda$. They proved bounds that characterize the
tradeoff between the quality of the approximation (decreasing in
$\lambda$), and the smoothness of the approximation (increasing in
$\lambda$).  This smoothing technique has also been applied to
Hamiltonian Monte Carlo~\citep{chaari2016hamiltonian}. Recently,
Durmus et al.~\citep{durmus2019analysis} established a mixing rate of
order $O(d / \varepsilon^2)$ for non-smooth composite objectives using
gradient flow in space of measures. However, their algorithm does not
directly lead to a Metropolis version suitable for the conductance
proof techniques, due to the singular measure yielded in the gproximal
step.

Note that all these previous works split the non-smooth component $g$
and the noise introduced in the sampling algorithm.  In contrast, we
take an alternative approach where the diffusion part and the
non-smooth function $g$ are combined together through a proximal
sampling oracle--- in particular, see
Definition~\ref{assume-oracle}). This joint approach leads to
significantly smaller bias within each step, and allows for uniform
control on the rejection probability.

\begin{comment}
\paragraph{Constraint sampling}
We also note that the problem of sampling from a density restricted to
a convex set has attracted far more attention; from the mixing time
$O(d^{23})$ of the lattice walk~\citep{dryer1991random} to the mixing
time $O(d^4)$) of the hit-and-run
walk~\citep{lovasz2007geometry}. Langevin based methods have also been
applied to this problem with the projected
Langevin~\citep{bubeck2018sampling}, Moreau-Yosida
Langevin~\citep{brosse2017sampling} and Mirrored
Langevin~\citep{hsieh2018mirrored}. But this case of non-Lipschitz
function $g$ is outside the scope of our work.  \mjwcomment{I don't
  think that this paragraph adds anything: we are not talking about
  constrained problems.  Suggest cutting}
\end{comment}


\paragraph{Basic definitions and notation:} Let us summarize some
definitions and notation used in the remainder of the paper.  The
Euclidean norm of a vector $x \in \real^d$ is denoted by
$\vecnorm{\cdot}{2}$.  We use $\mathcal L(X)$ to denote the law of a
random variable $X$. The total variation (TV) distance between two
distributions $\mathcal P$ and $\mathcal Q$ is given by
\mbox{$\DTV(\mathcal P,\mathcal Q) = \sup_{A \in \mathcal{B}
    (\real^d)} |\mathcal P(A) - \mathcal Q(A)|$.} Given an error
tolerance $\varepsilon>0$, we define the \emph{mixing time} associated
with the total variation distance of a Markov chain $X_t$ with
stationary distribution $\Pi$ as
\begin{align*}
\Tmix (\varepsilon) & \defn \arg \min_{k=1,2,\dots} \left \{
\DTV(\mathcal L(X_k),\Pi)\leq \varepsilon \right \}.\vspace{-0.21cm}
\end{align*}
The Kullback-Leibler divergence between two distributions is given by
$\kull{\mathcal P}{\mathcal Q}=\Exs_{\mathcal P}\left[\log
  \left(\frac{d\mathcal P}{d\mathcal Q}\right)\right]$.  In this
expression, the quantity $\frac{d \mathcal P}{d \mathcal Q}$ denotes
the Radon-Nikodym derivative of $\mathcal P$ with respect to $\mathcal
Q$.


\section{Metropolis-adjusted Proximal Algorithm}

We now describe the Metropolis-adjusted proximal algorithm that we
propose and study in this paper.

\subsection{Metropolis-Hastings algorithm}

The Metropolis-Hastings algorithm allows sampling in a simple and
efficient way from any target density $\pi$ known up to a
multiplicative constant.  For each $x \in \real^d$, let $p(x, \cdot)$
be a density from which it is relatively easy to sample, and for which
$p(x,y)$ is available up to a multiplicative constant independent of
$x$.  Each member of the family $\{p(x, \cdot), x \in \real^d \}$ is
known as a \emph{proposal distribution.}  The Metropolis-Hastings
algorithm associated with $p$ produces a discrete-time Markov chain
$\{X_t \}_{t \geq 0}$ in the following way: at each step a candidate
$y$ is proposed according to the density $p(x_t,\cdot)$ and is then
accepted with probability
\begin{align}
  \alpha(x_t,y) = \begin{cases}
    \min \left \{ \frac{\pi(y) p(y,x)}{ \pi(x_t) p(x,y)}, 1 \right
    \} & \mbox{if $\pi(x) p(x,y) > 0$, and} \\
    1 & \mbox{if $\pi(x) p(x,y)=0$.}
\end{cases}
\end{align}
Otherwise the candidate is rejected and the chain stays in its current
position.  The algorithm always accepts candidates $y$ when the ratio
$\pi(y)/p(x_t,y)$ is larger than the previous value
$\pi(x_t)/p(y,x_t)$ but may also accept candidates whose ratio is
smaller. The transition kernel of this Markov chain can be written as
\begin{align}
\label{eq:transition-metropolis-general}  
\transition_x(A) = \int_A p(x,y) \alpha(x,y) dy + \delta_x(A) \int (1
- \alpha(x,y)) p(x,y) dy,
\end{align}
from which it can be seen that the measure $\Pi$ is invariant for this
kernel.  Moreover, under the usual assumptions of aperiodicity and
irreducibility, the chain converges to the stationary distribution in
TV distance.  Various choices of proposal densities have been
investigated, such as the independence
sampler~\citep{mengersen1996rates}, the random walk, the Langevin
algorithm~\citep{roberts1996exponential}, or the symmetric
proposal~\citep{hastings1970monte}.


\subsection{Proximal proposal}

In this paper, we study a particular class of proposal distributions,
one designed to leverage the special structure of the density $\pi$.
\begin{definition}[Proximal sampling oracle]\label{assume-oracle}
 When queried with a vector $u \in \real^d$ and stepsize $\stepsize >
 0$, the oracle $\oracle_{\stepsize, g}(u)$ returns:
 \begin{enumerate}
 \item[(a)] a sample of a random variable $Y$ with density
   proportional to \mbox{$\exp \left( - \frac{1}{4 \stepsize}
     \vecnorm{y- u}{2}^2 - g(x)\right)$.}
\item[(b)] the value of the partition function $Z(u) \mydefn \int \exp
  \left( - \frac{1}{4 \stepsize} \vecnorm{y - u}{2}^2 - g(x)\right)
  dy$.
 \end{enumerate}
\end{definition}
The Metropolis-Hastings algorithm based on this oracle is given in
Algorithm~\ref{alg-prox}. Given the current iterate $x \in
\real^\dims$ and stepsize $\stepsize > 0$, it uses the oracle with $u
= x - \stepsize \nabla f(x)$ to draw a new sample $Y$ distributed as
\begin{align}\label{eq:proposal}
Y \sim p(x, \cdot) = {Z(x-\stepsize \nabla f (x))}^{-1} \exp \left( - \frac{1}{4 \stepsize}
\vecnorm{ \cdot - (x-\stepsize \nabla f (x))}{2}^2 - g(\cdot)\right).
\end{align}
Let $\proposal_x(\cdot)$ denote the distribution over $Y$ induced by
$p(x, \cdot)$, and let $\transition_x$ denote the transition
kernel~\eqref{eq:transition-metropolis-general}, both parameterized by
the centering point $x$.  We let
\begin{align*}
      \prej_x \mydefn \DTV (\proposal_x, \transition_x) = \Prob_{Y
        \sim \transition_x} (Y = x)
\end{align*}
be the probability that the proposal is rejected. Furthermore, let
$\transition^{succ}_x$ be the transition kernel conditionally on not
being rejected. It can be seen that $\transition^{succ}_x$ is
absolutely continuous with respect to Lebesgue measure, with density
given by
\begin{align*}
 \transition^{succ}_x (y) \propto {\min\left( p(x, y), e^{U(x) - U(y)}
   p(y, x) \right)}.
\end{align*}

\begin{algorithm}[h]
\caption{Metropolis-adjusted Proximal Gradient Langevin Dynamics}\label{alg-prox}
\begin{algorithmic}
\REQUIRE Access to $f, \nabla f, g, \oracle_{\stepsize, g} (z)$,
starting point $\xiter{0}$. Parameter $\stepsize$.  \ENSURE
Approximate sample from $p \propto e^{-U}$.  \STATE Sample $\xiter{0}
\sim \mathcal{N} (\referencepoint, \frac{1}{\smoothness + 1} I)$.
\FOR{$t = 0, 1, 2, \cdots$} \STATE Draw sample $Z \sim
\oracle_{\stepsize, g}( \xiter{t} - \stepsize \nabla f(\xiter{t}))$
using the Proximal sampling oracle.  \STATE $\xiter{t + 1}
= \begin{cases} Z & w.p.~ \min \left(1, \frac{e^{- U(z)} p(Z,
    \xiter{t})}{e^{- U(\xiter{t})} p(\xiter{t}, Z)}
  \right),\\ \xiter{t} & \text{otherwise}

    \end{cases}.$
\ENDFOR
\end{algorithmic}
\end{algorithm}
In the degenerate case when $g = 0$, the algorithm is essentially the
same as the Metropolis-adjusted Langevin (MALA) algorithm. For a
general function $g$, the one-step proposal
distribution~\eqref{eq:proposal} can be understood in the following
way: we approximate $f$ locally with a quadratic function, keep $g$
unchanged, and use this as a potential function. From the high-level
point of view, the proposal~\eqref{eq:proposal} is in many aspects
similar to proximal gradient methods. However, the analysis is not a
straightforward extension from the smooth case.

Compared to prior work, the key property of the proposed method---one
which leads to faster mixing guarantees--- is that it combines the
exact solver related to $g$ with the exact solver for the noise
part. This careful combination prevents the error in the Gaussian
noise part from being amplified by a discontinuous drift. We note that
this idea has been used to study proximal gradient descent for KL
divergence in the Wasserstein
space~\citep{bernton2018langevin,wibisono2018sampling}; notably, in
this setting, the one-step update is intractable by itself.  In sharp
contrast, it is possible to perform updates efficiently under our
one-step proposal distribution.



\section{Main results}

We now turn to our main results, beginning with our assumptions and a
statement of our main theorem in Section~\ref{SecStatement}, followed
by some examples for which drawing samples from the the proximal
proposal distribution is computationally efficient in
Section~\ref{SecExamples}.  We provide a high-level overview of the
proof in Section~\ref{SecOverview}.

\subsection{Statement of the main result}
\label{SecStatement}

We make the following assumptions:
\begin{assumption}[Logarithmic Sobolev inequality for $\pi$]
  \label{assume-log-sobolev}
The target density $\pi$ satisfies a log-Sobolev inequality with
constant $\logsobo > 0$, meaning that
\begin{align*}
\Exs_\pi[h(X) \log h(X)] & \leq  \frac{1}{2 \logsobo } \Exs_\pi
\left[ \tfrac{\vecnorm{\nabla h(X)}{2}^2}{h(X)} \right]
\end{align*}
for any Lipschitz function $h: \real^d \to \real $ with
$\Exs_\pi[h(X)] = 1$.
\end{assumption}

\begin{assumption}[Smooth function $f$]
  \label{assume-smooth}
There is a finite constant $\smoothness\geq0$ such that
\begin{align*}
\vecnorm{\nabla f (x)-\nabla f (y)}{2}\leq \smoothness\vecnorm{x-y}{2}
\quad \mbox{for all $x,y \in \real^d$.}
\end{align*}
\end{assumption}

\begin{assumption}[Distant dissipativity]
\label{assume-dissipative}
There exists a vector $x_0 \in \real^d$ and strictly positive
constants $\dissipative, \distantdissipative$ such that
\begin{align*}
  \inprod{\nabla f(x)}{x - \referencepoint} \geq
  \frac{\dissipative}{2} \vecnorm{x - x_0}{2}^2 - \distantdissipative
  \quad \mbox{for all $x \in \real^d$.}
\end{align*}
\end{assumption}
\noindent Note that this condition is a generalization of
$\dissipative$-strong convexity, which is a special case with $x_0$
corresponding to the global minimum of $f$, and $\distantdissipative =
0$.

\begin{assumption}[Convex and Lipschitz function $g$]
\label{assume-convex-lip-regularizer}
The function $g$ is convex, and there is a finite constant
$\lipschitzness>0$ such that
\begin{align*}
   |g(x)-g(y)| \leq \lipschitzness \vecnorm{x-y} \qquad \mbox{for all
     $x, y \in \real^d$.}
\end{align*}
\end{assumption}

Distributions satisfying a log-Sobolev
inequality~\citep{gross1975logarithmic} include strongly log-concave
distributions~\citep{bakry1985diffusions} as well as bounded
perturbations thereof~\citep{holley1987logarithmic}. These conditions
cover, for example, distributions that are strongly log-concave
outside a bounded region but non-log-concave inside; see the
paper~\citep{ma2018sampling} for some instances in the context of
mixture models. Note that the log-Sobolev inequality does not require
smoothness assumptions on the target density.  By the Lipschitz
condition (Assumption~\ref{assume-convex-lip-regularizer}) and
Rademacher's Theorem, the function $g$ is differentiable almost
everywhere (w.r.t.  Lebesgue measure).

For a given initial vector $\referencepoint \in \real^\dims$ and
tolerance parameter $\varepsilon > 0$, we define
\begin{align}
A_0 \mydefn \vecnorm{\nabla U(\referencepoint)}{2}, \quad \mbox{and }
\textstyle R \mydefn \frac{C}{\sqrt{\dissipative}} \; \sqrt{
  (\smoothness + \frac{1}{\dissipative}) A_0^2 + \lipschitzness^2 +
  \distantdissipative + d \log \frac{4 (\smoothness + 1)}{
    \dissipative} + \log \left(\frac{1}{\varepsilon} \right)}.
  \end{align}
With these definitions, we have
\begin{theorem}
  \label{thm-prox-mixing-rate}
Suppose that Assumptions 1--4 hold, and moreover, that $\max
(\dissipative, \smoothness, \lipschitzness, \distantdissipative,
\vecnorm{A_0}{2})$ grows at most polynomially in dimension $d$.  Then
there is a universal constant $C'$ such that Algorithm~\ref{alg-prox}
with stepsize $\stepsize = \frac{1}{2 \smoothness^2 R^2 + \smoothness
  d}$ has mixing time bounded as
\begin{align}
  \label{EqnProxMixingRate}
\Tmix(\varepsilon) \leq \frac{C' \smoothness^2}{\logsobo \dissipative}
\left \{ \left(\smoothness + \frac{1}{\dissipative} \right) A_0^2 +
\lipschitzness^2 + \distantdissipative + d \log \left( \frac{4
  (\smoothness + 1)}{\dissipative} \right) + \log(1/\varepsilon)
\right \} \log \left(\frac{d}{\varepsilon} \right).
\end{align}
\end{theorem}
\noindent See Section~\ref{SecOverview} for a high-level overview of
the proof of Theorem~\ref{thm-prox-mixing-rate}. The full argument is
given in Section~\ref{sec:proofs}. \\

In order to interpret the mixing time bound~\eqref{EqnProxMixingRate},
it is helpful to consider some particular settings of the problem
parameters. Suppose that $f$ is $\dissipative$-strongly convex with
condition number $\kappa \mydefn \frac{\smoothness}{\dissipative}$ and
that $g$ is Lipschitz with parameter $\lipschitzness = O(\sqrt{d})$;
for example, if $g$ is chosen to a Laplacian prior (see the next
section for details), this latter Lipchitz condition will hold.
Taking $\referencepoint$ to be the minimizer\footnote{The minimizing
  vector $\referencepoint$ can be efficiently computed in
  $O(\sqrt{\kappa} \log \varepsilon^{-1})$ time independent of
  dimension using accelerated gradient methods. Moreover, for a
  strongly convex function, Assumption~\ref{assume-dissipative} holds
  with the same value of $\mu$ and $\beta = \lipschitzness^2$.}  of
$U$ ensures that $A_0 = 0$, and thus, the mixing rate scales as
$\Tmix(\varepsilon) = O \left(\kappa^2 (d \log (d/\varepsilon) +
\log^2 (1/\varepsilon)) \right)$.  Up to an extra multiple of the
condition number $\kappa$, this rate matches the best known guarantee
for the MALA algorithm in the smooth case~\citep{dwivedi2018log}.  In
contrast, the best prior work for nonsmooth problems requires
$O(d/\varepsilon^2)$ iterations and gradient
evaluations~\citep{durmus2019analysis}, so that our method leads to
exponentially faster convergence while retaining the same dimension
dependency.


\subsection{Examples of Proximal Sampling Oracles}
\label{SecExamples}

We describe here examples of functions $g$ for which the associated
proximal proposal can be implemented in a computationally efficient
manner.

\paragraph{Coordinate-separable regularizers:}

Consider a regularizer that is of the the coordinate-separable form $g
\defn \sum_{i=1}^dg_i(x_i)$.  In this case, the proposal distribution
can be factorized as
\begin{align*}
p(x,y) & = \prod_{i=1}^d p_i(x_i,y_i), \quad \mbox{where
  $p_i(x_i,y_i)= \frac{1}{Z_i} e^{-\frac{1}{4\stepsize} (y_i-x_i)^2
    -g_i(y_i)}$.}
\end{align*}
Sampling from the proposal distribution thus reduces to a collection
of $d$ univariate sampling problems.  (Note that the original problem
of sampling from $\pi$ will still be a genuinely $d$-variate problem
whenever $f$ is \emph{not} is coordinate-separable.)  Since each $p_i$
is a one-dimensional log-concave distribution, sampling can be
performed using black-box rejection style
algorithms~\citep{devroye1986non,gilks1992adaptive,devroye2012note}
and the partition function $Z_i$ can be computed using adaptive
methods for numerical integration, including numerical libraries such
as QUADPACK~\citep{piessens1983quadpack}. In this way, the overall
complexity of the oracle is still $O(d)$---the same order as the usual
gradient computation.

Our preceding discussion applies to a generic coordinate-separable
function $g$.  Closed-form expressions can be obtained for specific
functions $g$, such as in the following example.

\paragraph{$\ell_1$-regularization} The 
Bayesian Lasso~\citep{park2008bayesian} is based on the Laplace prior,
with log density $g_i(x_i)=\lambda |x_i|$.  In this case, the
partition function takes the form $ Z_i= {\sqrt{\pi
    \stepsize}}(\alpha_++\alpha_-)$ where $\alpha_\pm=e^{\pm\lambda
  z^2}\left[1\mp
  \text{erf}\left(\frac{(1\pm2\stepsize\lambda)y_i}{2\sqrt{\stepsize}}\right)\right]$.
Here $\text{erf}$ denotes the Gaussian error function, and the random
variable $Y_i$ is drawn according to the mixture distribution
\begin{align*}
Y_i\sim \alpha_+ \mathcal{TN}_{(-\infty,0] }((1+2\stepsize\lambda)y_i,
  2\stepsize) + \alpha_-
  \mathcal{TN}_{[0,\infty)}((1-2\stepsize\lambda)y_i, 2\stepsize),
\end{align*}
where $\mathcal{TM}_{[a,b]}(\mu,\sigma^2)$ is the truncated normal
distribution. Drawing samples of $Y$ can be performed using fast
sampling methods for the truncated
Gaussian~\citep{Cho11,botev2017normal}. The Laplace prior can also be
combined with a Gaussian prior to obtain the Bayesian
Elastic-net~\citep{LiLin1}, to which our methodology applies in an
analogous way.

\paragraph{Group Lasso} The group Lasso is a generalization of the
Lasso method where the features are grouped into disjoint blocks
$\{x_1, \ldots, x_G\}$. The penalty considered is $\sum_{j=1}^G
\vecnorm{x_j}{2}$. It is able to do variable selection at the group
level and corresponds to Multi-Laplacian
priors~\citep{raman2009bayesian}. The proximal sampling oracle can be
decomposed into product measure of groups, with each group sampling
from density $p(x, \cdot) \propto
\exp(-\frac{\vecnorm{\cdot-x}{2}^2}{4\stepsize}-\vecnorm{\cdot}{2})$. For
any vector $y \in \real^d$, let $y = a x + z$ with $z \perp x$ be its
orthogonal decomposition. Note that $p(x, \cdot)$ is symmetric around
axis $x$, so let $r = \vecnorm{z}{2}$. We can first solve the
two-dimensional sampling problem\footnote{This can be done by
  inverting the conditional and marginal CDF.} with density $q (r, a)
= r^{d - 2} \exp(-\frac{(a - 1)^2 \vecnorm{x}{2}^2 + r^2}{4\stepsize}
- \sqrt{r^2 + a^2})$, then draw a independent unit vector $ v \sim
\mathcal{U}(\sphere^{d - 2})$ and finally construct the proposal by $a
x + r v$.


\subsection{Proof Overview}
\label{SecOverview}

We now provide a high-level overview of the main steps involved in the
proof of Theorem~\ref{thm-prox-mixing-rate}.  First of all, the
Metropolis filter automatically guarantees that the Markov chain
defined by the kernel $\transition$ has $\Pi$ as its stationary
distribution.  By Assumption~\ref{assume-log-sobolev}, the underlying
density satisfies a Gaussian isoperimetric
inequality~\citep{bakry1996levy,bobkov1999isoperimetric}. Using the
known framework for conductance and mixing of Markov
chains~\citep{goel2006mixing,kannan2006blocking,chen2019hmc}---to be
reviewed in Section~\ref{app:markov}---we only need to show the
following two facts hold over a sufficiently large ball $\Omega
\subseteq \real^d$ enclosing most of the mass of $\Pi$:
\begin{itemize}
\item[\textbf{Fact 1:}] Rejection happens only with constant
  probability: namely, there is a universal constant $c \in [0,1)$
    such that $\DTV (\proposal_x, \transition_x) \leq c$ \mbox{for all
      $x \in \Omega$.}
\item[\textbf{Fact 2:}] The transition kernels $\transition$ at two
  neighboring points are close: namely, there exist positive scalars
  $\omega, \Delta > 0$ such that for $x, y$ with $\vecnorm{x - y}{2}
  \leq \Delta$, we have $\DTV \left(\transition_x,
  \transition_y\right) \leq 1 - \omega$.
\end{itemize}

For an initial distribution $\Pi_0$ with an initial condition
$\warmstart \mydefn \sup_{x \in \Omega} \frac{\pi_0 (x)}{\pi (x)}$,
the mixing rate can then be upper bounded as
\begin{align*}
    \Tmix (\varepsilon) \lesssim \log \warmstart + \frac{1}{\omega^2
      \logsobo \Delta^2} \left( \log \varepsilon^{-1} + \log \log
    \warmstart \right).
\end{align*}
See Section~\ref{app:proofmainprop} for the details of this argument.
We note that an initial vector $x_0$ for which $\warmstart = e^{O(d)}$
can be achieved by Gaussian initialization (see
Section~\ref{sec:warm}).  Let us now provide high-level sketches of
the proofs of Facts 1 and 2, respectively.


\subsubsection{Fact 1: Analysis of the Rejection probability}

In Section~\ref{app:rejproba}, we establish the following bound on the
rejection probability:
\begin{lemma}
\label{lemma-acc-rej}
Under Assumptions~\ref{assume-smooth}, \ref{assume-dissipative},
and~\ref{assume-convex-lip-regularizer}, there is a universal positive
constant $C$ such that for any stepsize $\stepsize \in \left(0,
\frac{1}{16 (\smoothness + 1)} \right)$ and for any $x \in \real^d$,
the TV distance is bounded as
\begin{align*}
  \DTV\left( \proposal_x, \transition_x \right) \leq \frac{1}{2} + C
  \stepsize \left( \smoothness d + \lipschitzness^2+ \vecnorm{\nabla
    f(x)}{2}^2 \right).
\end{align*}
\end{lemma}
\noindent The core of the proof involves upper bounding the integral
\begin{align*}
\int p(x, z) \max \left(0, 1\!-\!\frac{e^{- U(z)} p(z,
  x)}{e^{- U(x)} p(x, z)} \right) dz.
\end{align*}
A straightforward calculation yields
\begin{align*}
\frac{e^{- U(z)} p(z, x)}{e^{- U(x)} p(x, z)} = \frac{Z(x)}{Z(z)} \exp
\left( f(x) - f(z) - \frac{1}{4 \stepsize} \vecnorm{x - z + \stepsize
  \nabla f(z)}{2}^2 + \frac{1}{4 \stepsize} \vecnorm{z - x + \stepsize
  \nabla f(x)}{2}^2 \right).
\end{align*}

Note that the terms $g(x)$ and $g(z)$ in the exponent cancel out when
comparing the proposal distribution with the target
density. Completing the proof then requires two steps: (a) lower
bounding the ratio $\frac{Z(x)}{Z(z)}$ of partition functions and (b)
lower bounding the exponential factor involving $f$ and its gradients.

At a high level, the proof of step (b) is relatively routine, similar
in spirit to analysis due to Dwivedi et al.~\citep{dwivedi2018log}. We
decompose the exponent into two error terms in first-order Taylor
expansion $f(z) - f(x) - \inprod{z - x}{\nabla f(x)}$ and $f(x) - f(z)
- \inprod{x - z}{\nabla f(z)}$, and a term of the form
$\vecnorm{\nabla f(x)}{2}^2 - \vecnorm{\nabla f(z)}{2}^2$. If the
distance from $x$ to the proposal $z$ can be controlled, we can easily
upper bound the three terms by Assumption~\ref{assume-smooth} alone,
without using convexity.

Proving the claim in step (a), however, is highly non-trivial.  The
partition function $Z$ can be seen as a smoothed version of the
function $e^{-g}$. Intuitively, a sample $u$ drawn from the proposal
distribution centered at $x$ will be dispersed around $x$, with
roughly half of the directions increasing the value of the partition
function. So with probability approximately one half, we expect that
$Z(u)$ is not much larger than $Z(x)$.


\subsubsection{Fact 2: Overlap bound for transitions kernels}

Note that the rejection probability bound proved in
Lemma~\ref{lemma-acc-rej} can only guarantee that the proposal point
is accepted with probability arbitrarily close to
$\tfrac{1}{2}$. Therefore, in contrast to the past work of Dwivedi et
al.~\citep{dwivedi2018log} on MALA, in order to obtain non-trivial
bounds on $\DTV (\transition_{x_1}, \transition_{x_2})$, it no longer
suffices to control $\DTV (\proposal_{x_1}, \proposal_{x_2})$ and
apply the triangle inequality.

Instead, we directly bound the total variation distance between the
transition kernels at two neighboring points. In particular, via a
direct calculation, we show that
\begin{align}
\label{eq:tv-decomposition}  
  \DTV \left( \transition_{x_1}, \transition_{x_2} \right) \leq \max
  (\prej_{x_1}, \prej_{x_2}) + \DTV \left( \transition^{succ}_{x_1},
  \transition^{succ}_{x_2} \right) + |\prej_{x_1} - \prej_{x_2}|.
\end{align}
See Section~\ref{app:overlap} for the proof of this bound.

Overall, the bound for $\DTV \left( \transition_{x_1},
\transition_{x_2} \right)$ consists of three parts: the first term
directly comes from Lemma~\ref{lemma-acc-rej}; the second term is the
TV distance between the kernels conditioned on successful transitions;
and the last term is the difference between rejection
probabilities. Upper bounds for the latter two terms are proven in
Lemma~\ref{lemma-overlap-succ-transition} and
Lemma~\ref{lemma:diff-rejection-prob}, respectively, which we state
here.

\begin{lemma}
  \label{lemma-overlap-succ-transition}
Suppose that Assumptions~\ref{assume-smooth},~\ref{assume-dissipative}
and~\ref{assume-convex-lip-regularizer} hold, and consider a step size
$\stepsize \in \left(0, \frac{1}{16(\smoothness + 1)}\right) $.  Then
for any $x_1, x_2 \in \real^\dims$, we have
\begin{align}
  \DTV \left( \transition^{succ}_{x_1}, \transition^{succ}_{x_2}
  \right) \leq 5 \sqrt{ \vecnorm{x_1 - x_2}{2} \cdot \left(
    \vecnorm{\nabla f(x_1)}{2} + \vecnorm{\nabla f(x_2)}{2} +
    \lipschitzness + \smoothness \sqrt{\stepsize d} \right)} + 2
  \frac{\vecnorm{x_1 - x_2}{2}}{\sqrt{\stepsize}}.
\end{align}
\end{lemma}
\begin{restatable}{lemma}{diffrejectionprob}
\label{lemma:diff-rejection-prob}
Suppose that Assumptions~\ref{assume-smooth}
and~\ref{assume-convex-lip-regularizer} hold, and consider a stepsize
$\stepsize \in \left(0, \frac{1}{16(\smoothness + 1)}\right) $.  Then
there is a universal constant $C > 0$ such that for any $x_1, x_2 \in
\real^\dims$, we have
\begin{align}
  |\prej_{x_1} - \prej_{x_2}|\leq 2 \frac{\vecnorm{x_1 -
      x_2}{2}}{\sqrt{\stepsize}} + C \vecnorm{x_1 - x_2}{2} \left(
  \sup_{0 \leq \lambda \leq 1}\vecnorm{\nabla f((1 - \lambda) x_1 +
    \lambda x_2)}{2} + \lipschitzness + \smoothness \sqrt{\stepsize d}
  \right).
\end{align}
\end{restatable}

By Lemma~\ref{lemma-acc-rej}, the choice of step size parameter
$\stepsize = O(1 / d)$ suffices to make the first term in
equation~\eqref{eq:tv-decomposition} less than $\frac{7}{10}$, The
final two terms in equation~\eqref{eq:tv-decomposition} can be made
less than $\tfrac{1}{10}$ using
Lemma~\ref{lemma-overlap-succ-transition} and
Lemma~\ref{lemma:diff-rejection-prob}, with $\vecnorm{x_1 - x_2}{2}
\leq \frac{1}{\sqrt{\stepsize}}$. Putting together these guarantees
ensures that $\DTV (\transition_{x_1}, \transition_{x_2}) \leq
\frac{9}{10}$. See Proposition~\ref{corr-overlap-transition} in
Section~\ref{app:overlap} for a precise statement of this claim.


\section{Proofs}
\label{sec:proofs}

In this section, we provide the full details of our proof.
Section~\ref{app:proofth1} is devoted to the proof of
Theorem~\ref{thm-prox-mixing-rate} given necessary lemmas controlling the rejection probability and overlap bounds. We prove the rejection probability
bound stated in Lemma~\ref{lemma-acc-rej} in
Section~\ref{app:rejproba}, and the overlap bound transition kernels
of Lemma~\ref{lemma-overlap-succ-transition} and
Lemma~\ref{lemma:diff-rejection-prob} in
Section~\ref{app:overlap}. The tail bounds needed in our analysis are
postponed to Appendix~\ref{app:tail}.


\subsection{Proof of Theorem~\ref{thm-prox-mixing-rate}}
\label{app:proofth1}

Taking Lemma~\ref{lemma-acc-rej} and
Proposition~\ref{corr-overlap-transition} as given, let us now prove
Theorem~\ref{thm-prox-mixing-rate}.  We first introduce some known
results on continuous-space Markov chain mixing based on conductance
and isoperimetry, and then use them to prove the theorem. An upper
bound for the warmness parameter in feasible start is needed in the
proof, which is established in Section~\ref{sec:warm}.


\subsubsection{Some known results}
\label{app:markov}

Our analysis makes use of mixing time bounds based on the conductance
profile, given by
\begin{align*}
  \Phi_\Omega(v) \mydefn \inf_{0 \leq \pi(S \cap \Omega) \leq v}
  \frac{\int \transition^{succ}_x (S^c) d\pi(x)} { \pi(S \cap \Omega)}
  \qquad \mbox{for any $v \in (0, \frac{\pi(\Omega)}{2})$.}
\end{align*}
The following result~\cite{kannan2006blocking,chen2019hmc} uses the
conductance profile to bound the mixing time of a reversible,
irreducible $\tfrac{1}{2}$-lazy Markov chain with transition
distribution absolutely continuous with respect to Lebesgue measure.
\begin{proposition}
  \label{prop:avgconductance-to-mixing}
For a given error rate $\varepsilon \in (0,1)$ and warm start
parameter $\warmstart$, suppose there is a set $\Omega \subseteq
\real^d$ such that $\Pi(\Omega) > 1 - \frac{\varepsilon^2}{2
  \warmstart^2}$. Then the $L^2$-mixing time from any
$\warmstart$-warm start is bounded as
  \begin{align}
\label{eq:mixing_bound_using_conductanceprofile}
    \Tmix(\varepsilon; \warmstart, \ell_2) & \leq
    \int_{4/\warmstart}^{\frac{\Pi (\Omega)}{2}} \frac{8 dv}{ v
      \Phi_\Omega^2 (v)} + \frac{8}{ \Phi^2_\Omega \big(\Pi (\Omega)/2
      \big)} \; \log \left( \frac{16}{\varepsilon \Pi (\Omega)}
    \right).
  \end{align}
\end{proposition}

We also need the following classical result that relates the
conductance of a continuous-state Markov chain with the isoperimetric
inequality of the target measure and overlap bound of transition
kernels~\cite{kannan2006blocking,chen2019hmc}.  In particular, we say
that the transition kernels satisfy an overlap bound with parameters
$\omega, \Delta$ over a set $\Omega$ if for any pair $x, y \in \Omega$
such that $\vecnorm{x-y} {2} \leq \Delta$, we have
$\DTV(\transition_x, \transition_y) \leq 1 - \omega$.

\begin{proposition}
\label{prop:isoperi+overlap=conductance}
Consider a Markov chain with a target distribution $\pi$ that is
absolutely continuous with respect to Lebesgue measure, satisfies the
Gaussian isoperimetric inequality with constant $\logsobo$, and such
that its transition distribution satisfies the $(\Delta,
\omega)$-overlap condition. Then for $s \in (0,1)$ and any convex
measurable set $\Omega$ such that $\pi(\Omega) \geq 1-s$, we have
  \begin{align}
  \label{eq:conductanceprofile_via_overlaps}
    \Phi_\Omega (v) \geq \frac{\omega}{4} \cdot \min \left(1, \frac{
      \Delta \sqrt{\logsobo}}{16} \cdot \log^{1/2} \left(1 +
    \frac{1}{v} \right) \right) \quad \mbox{for all $v \in
      \brackets{0, \frac{1-s}{2}}$.}
  \end{align}
\end{proposition}


\subsubsection{Proof of Theorem~\ref{thm-prox-mixing-rate}}
\label{app:proofmainprop}

Let now turn to the proof of Theorem~\ref{thm-prox-mixing-rate}, which
involves establishing a lower bound for $\Phi_{\Omega}$ using
Proposition~\ref{prop:isoperi+overlap=conductance}.  Combining this
lower bound with Proposition~\ref{prop:avgconductance-to-mixing}
yields the final mixing rate bound.

First of all, by Lemma~\ref{cor-tail-target}, letting $\Omega \mydefn
\ball (\referencepoint, R_s)$ with $R_s = C
\sqrt{\frac{\distantdissipative + d + \lipschitzness^2 + \log
    s^{-1}}{\dissipative}}$, we have $\pi (\Omega) > 1 - s$. The
discussion about conductance can be restricted to $\Omega$. The
parameter $s$ will be chosen later.  By
Assumption~\ref{assume-log-sobolev}, the target distribution $\pi$
satisfies a log-Sobolev inequality with constant $\logsobo$, which
implies a Gaussian isoperimetric inequality with constant $\logsobo$,
due to~\cite{bakry1996levy}.  Choosing the step size to be $\stepsize
= \frac{c}{ \lipschitzness^2 + \smoothness^2 (A_0^2 + R_s^2) +
  \smoothness d}$ ensures the following properties:
\begin{itemize}
    \item According to Lemma~\ref{lemma-acc-rej}, for any $x \in
      \Omega$,
    \begin{align*}
      \DTV (\proposal_x, \transition_x) \leq \frac{1}{2} + C
      \stepsize (\smoothness d + \lipschitzness^2 + \vecnorm{\nabla
        f(x)}{2}^2) \leq \frac{1}{2} + C \stepsize (\smoothness d +
      \lipschitzness^2 + (A_0 + \smoothness R_s)) \leq \frac{2}{3}.
    \end{align*}
    Consequently, the number of attempts required for a successful
    transition is a geometric random variable with rate at least
    $1/3$. With high probability, the number of successful transitions
    is of the same order as the number of steps.
    \item According to Proposition~\ref{corr-overlap-transition}, for
      any pair $x_1, x_2 \in \Omega$ such that
\begin{align*}
  \vecnorm{x_1 - x_2}{2} \leq \Delta \mydefn c' \min \left(
  \sqrt{\stepsize}, \frac{1}{A_0 + \smoothness R_s + \lipschitzness}
  \right) = c' \sqrt{\stepsize},
    \end{align*}
 we are guaranteed that $\DTV (\transition_{x_1}, \transition_{x_2})
 \leq \frac{9}{10} \mydefn 1 - \omega$.
\end{itemize}
By Proposition~\ref{prop:isoperi+overlap=conductance}, we obtain:
\begin{align*}
\Phi_\Omega(v) & \geq \frac{1}{40} \min \left(1, \frac{
  \sqrt{\stepsize \logsobo}}{16} \log^{1/2}\left( 1 + \frac{1}{v}
\right)\right),\quad \forall v \in [0, (1 - s)/2].
\end{align*}
As shown $\warmstart = \sup_{x \in \Omega} \frac{\pi_0 (x)}{ \pi (x)}$
has an upper bound independent of $R_s$. In order to apply
Proposition~\ref{prop:avgconductance-to-mixing}, we need $s \leq
\frac{\varepsilon^2}{2 \warmstart^2}$, which means:
\begin{align*}
    \log s^{-1} \geq (\smoothness + \dissipative^{-1}) A_0^2 +
    \frac{\lipschitzness^2}{4} + \distantdissipative + \frac{d}{2}
    \log \frac{4 (\smoothness + 1)}{ \dissipative} + 2 \log
    \frac{1}{\varepsilon}.
\end{align*}
Set $s$ to this value, and set
\begin{align*}
    R_s = C \sqrt{ \frac{(\smoothness + \dissipative^{-1}) A_0^2 +
        \lipschitzness^2 + \distantdissipative + d \log \frac{4
          (\smoothness + 1)}{ \dissipative} + 2 \log
        \varepsilon^{-1}}{\dissipative} }.
\end{align*}
Substituting this expression back into the integral in
Proposition~\ref{prop:avgconductance-to-mixing} yields
\begin{align*}
\Tmix(\varepsilon; \warmstart, \ell_2) \leq &
\int_{4/\warmstart}^{\frac{\pi (\Omega)}{2}} \frac{8 dv}{ v
  \Phi_\Omega^2 (v)} + \Phi_\Omega \left(\frac{\pi
  (\Omega)}{2}\right)^{-2} \int_{\frac{\pi
    (\Omega)}{2}}^{\frac{8}{\varepsilon}} \frac{8dv}{v} \\
\leq & \int_{4 / \warmstart}^{1} \frac{dv}{v} + \frac{C}{\stepsize
  \logsobo}\int_{4/\warmstart}^{\frac{\pi (\Omega)}{2}} \frac{ dv}{ v
  \log \frac{1}{v}} + \frac{C}{\stepsize \logsobo} \log
\frac{1}{\varepsilon} \\
\leq & \log \warmstart + \frac{C}{\stepsize \logsobo} \left( \log \log
\warmstart + \log \frac{1}{\varepsilon}\right).
\end{align*}
Assuming that $\max (\dissipative, \smoothness, \lipschitzness,
\distantdissipative, \vecnorm{A_0}{2}) \leq \mathrm{poly}(d)$, we
have:
\begin{align*}
\Tmix(\varepsilon; \warmstart, \ell_2) \leq \frac{C}{\stepsize
  \logsobo} \log \frac{d}{\varepsilon},
\end{align*}
which finishes the proof.


\subsubsection{Upper Bound for Warmness with Feasible Start}
\label{sec:warm}

In the following, we provide an coarse upper bound for the
``warmness'' constant of the Markov chain, using the initial
distribution defined by Algorithm~\ref{alg-prox}. By direct
computation, we obtain:
\begin{align*}
    \warmstart & = \sup_{x \in \Omega} \frac{\pi_0 (x)}{\pi (x)} \leq
    \left( \frac{1 + \smoothness}{ 2 \pi} \right)^{ \frac{d}{2}}
    \left(\int e^{- f(x) - g(x)} dx\right) \sup_{x \in \Omega} \exp
    \left( - (1 + \smoothness) \vecnorm{x - \referencepoint}{2}^2 +
    f(x) + g(x) \right) \\
& \leq {\left( (1 + \smoothness)/ (2 \pi) \right)^{
    \frac{d}{2}}} \left(\int e^{- f(x) - g(x)} dx \right) \\
&  \qquad \; \times \; \sup_{x \in \Omega} \exp \left( - (1 +
    \smoothness) \vecnorm{x - \referencepoint}{2}^2 + \smoothness
    \left( \vecnorm{x - \referencepoint}{2}^2 + \vecnorm{A_0}{2}^2 +
    \vecnorm{x - \referencepoint}{2}^2 + \frac{1}{4}\lipschitzness^2 +
    U(\referencepoint) \right) \right) \\
   & \leq {\left( (1 + \smoothness)/ (2 \pi) \right)^{ \frac{d}{2}}}
    \left(\int e^{- f(x) - g(x)} dx \right)\exp \left( L
    \vecnorm{A_0}{2}^2 + \frac{1}{4} \lipschitzness^2 + U(x_0)
    \right).
\end{align*}
Note that by Assumption~\ref{assume-dissipative} and
Assumption~\ref{assume-convex-lip-regularizer}, for $v \in \partial g
(\referencepoint)$ we have:
\begin{align*}
 \int e^{- f(x) - g(x)} dx \leq& \int \exp \left( - U
 (\referencepoint) - \frac{\dissipative}{2} \vecnorm{x -
   \referencepoint}{2} + \distantdissipative + \inprod{\nabla f(x_0) +
   v}{ x - \referencepoint} \right) dx \\
& \leq e^{- U(\referencepoint) } \int \exp \left( \distantdissipative
 + \frac{A_0^2}{\dissipative} - \frac{\dissipative}{4} \vecnorm{x -
   \referencepoint}{2} \right) dx = e^{- U(\referencepoint) +
   \distantdissipative + \frac{A_0^2}{\dissipative} } \left(
 \dissipative / (8 \pi) \right)^{- \frac{d}{2}}.
\end{align*}
Substituting back into the equation above yields
\begin{align*}
    \warmstart \leq \left( \frac{4 (\smoothness + 1)}{ \dissipative}
    \right)^{\frac{d}{2}} \exp \left( (\smoothness +
    \dissipative^{-1}) \vecnorm{A_0}{2}^2 + \frac{\lipschitzness^2}{4}
    + \distantdissipative \right).
\end{align*}


\subsection{Analysis of the rejection probability}
\label{app:rejproba}

This section is devoted to analysius of the rejection probability, and
in particular, the proof of Lemma~\ref{lemma-acc-rej}. Several
auxiliary results are needed in the proof, which are established in
the second subsection.


\subsubsection{Proof of Lemma~\ref{lemma-acc-rej}}
\label{app:lemma-acc-rej}

By definition, we have $\DTV \left( \proposal_x, \transition_x \right)
= \int p(x, z) \max \left(0, 1 - \frac{e^{- U(z)} p(z, x)}{e^{- U(x)}
  p(x, z)} \right) dz$.  Note that:
\begin{align*}
  &\frac{e^{- U(z)} p(z, x)}{e^{- U(x)} p(x, z)} \\ =&
  \frac{Z(x)}{Z(z)}\exp \left( - (f(z) + g(z)) + (f(x) + g(x)) -
  \frac{1}{4 \stepsize} \vecnorm{x - z + \stepsize \nabla f(z)}{2}^2 +
  \frac{1}{4 \stepsize} \vecnorm{z - x + \stepsize \nabla f(x)}{2}^2 -
  g(x) + g(z) \right)\\ =& \frac{Z(x)}{Z(z)}\exp \left( - (f(z) -
  f(x)) - \frac{1}{4 \stepsize} \vecnorm{x - z + \stepsize \nabla
    f(z)}{2}^2 + \frac{1}{4 \stepsize} \vecnorm{z - x + \stepsize
    \nabla f(x)}{2}^2 \right),
    \end{align*}
where $Z(y) \mydefn \int \exp\left( - \frac{1}{4 \stepsize} \vecnorm{q
  - y}{2}^2 - g(q) \right) dq$ for any $y \in \real^d$.
    
Let
\begin{align*}
Q_1(x,z) & = \frac{Z(x)}{ Z(z)}, \quad \mbox{and} \\
Q_2(x, z) & = \exp \left( -
(f(z) - f(x)) - \frac{1}{4 \stepsize} \vecnorm{x - z + \stepsize
  \nabla f(z)}{2}^2 + \frac{1}{4 \stepsize} \vecnorm{z - x + \stepsize
  \nabla f(x)}{2}^2 \right).
\end{align*}
With these choices, we have:
 \begin{align}
   \label{eq:dtvprop}
   \DTV(\proposal_x, \transition_x) \leq & \int p(x,z)
   \max\left(0, 1 - Q_1(x, z) Q_2(x, z)\right) dz
   \nonumber\\ \overset{(i)}{\leq}& \int p(x,z) \max\left(0, 1 -
   Q_1(x, z)\right) dz + \int p(x,z) \max\left(0, 1 - Q_2(x,
   z)\right) dz,
 \end{align}
 where step (i) follows from the elementary inequality
\begin{align*}
 \max (0, 1 - a b) \leq \max(0,1-a)+\max(0,1-b) \qquad \mbox{valid for
   $a, b \geq 0$.}
\end{align*}

In the following, we bound $Q_1$ and $Q_2$ from below.  Introducing
the convenient shorthand \mbox{$G(y) \mydefn - \log Z(y)$,} we have:
\begin{align*}
\nabla G(y) =  - \frac{\nabla Z(y)}{Z(y)} & = - \frac{\int \frac{q -
    y}{2 \stepsize} \exp\left( - \frac{1}{4 \stepsize} \vecnorm{q -
    y}{2}^2 - g(q) \right) dq}{\int \exp\left( - \frac{1}{4 \stepsize}
  \vecnorm{q - y}{2}^2 - g(q) \right) dq} \\
& \overset{(i)}{=} Z(y)^{-1}\int \nabla g(q) \exp\left( - \frac{1}{4
  \stepsize} \vecnorm{q - y}{2}^2 - g(q) \right) dq \\
& = \Exs_{q \sim \proposal_y} \nabla g(q),
 \end{align*}
where step (i) follows via integration by parts.  Putting together the
pieces yields the Lipschitz continuity of $G$:
\begin{align*}
  \vecnorm{\nabla G(y)}{2} \leq Z(y)^{-1}\int \vecnorm{ \nabla
    g(q)}{2} \exp\left( - \frac{1}{4 \stepsize} \vecnorm{q -
    y}{2}^2 - g(q) \right) dq \leq \lipschitzness.
\end{align*}

Note that $Z(\cdot)$ is actually the convolution between $e^{-g}$ and
the Gaussian density---viz $Z(y) = e^{-g} *
e^{-\frac{\vecnorm{\cdot}{2}^2}{4 \stepsize}}(y)$.  As a consequence
of the Pr\'{e}kopa-Leindler inequality
(e.g.~\cite{wainwright2019high}, Chapter 3), the function $Z$ is
log-concave, and the function $G$ is convex.

By Lemma~\ref{lemma-expg-coupling}, for any fixed $y \in \real^\dims$,
there exists a coupling $\gamma$ such that for $(X_1, X_2) \sim
\gamma$, we have $X_1, X_2 \sim \oracle_{\stepsize, g}(y)$, and
$\vecnorm{\frac{X_1 + X_2}{2} - y}{2} \leq \stepsize \lipschitzness$
almost surely. Therefore, for $Z \sim \oracle_{\stepsize, g}(y)$ we
have:
\begin{align}
 \label{eq:glip}
 \Prob \left( G(Z) \geq G(y) - \stepsize \lipschitzness^2 \right) = &
 \frac{1}{2} \left( \Prob \left( G(X_1) \geq G(y) - \stepsize
 \lipschitzness^2 \right) + \Prob \left( G(X_2) \geq G(y) - \stepsize
 \lipschitzness^2 \right)\right) \nonumber \\ \overset{(i)}{\geq} &
 \frac{1}{2} \Prob \left( \max(G(X_1), G(X_2)) \geq G(y) - \stepsize
 \lipschitzness^2 \right) \nonumber \\
 \overset{(ii)}{\geq} & \frac{1}{2} \Prob \left( G\left( (X_1 +
 X_2)/2\right) \geq G(y) - \stepsize \lipschitzness^2 \right)\nonumber
 \\
 \overset{(iii)}{\geq} & \frac{1}{2} \Prob \left( \vecnorm{\left( (X_1
   + X_2)/2 \right) - y }{2} \leq \stepsize \lipschitzness \right)
 \overset{(iv)}{=} \frac{1}{2},
\end{align}
where step (i) follows from the union bound; step (ii) uses the
convexity of $G$; step (iii) exploits the Lipschitzness of $G$; and
step (iv) is a direct consequence of Lemma~\ref{lemma-expg-coupling}.
    
Therefore, using equation~\eqref{eq:glip} with $y=x-\stepsize \nabla
f(x)$, we find that
\begin{align*}
  \Prob_{z \sim \proposal_x} \Big[ Q_1(x, z) \geq e^{ -\stepsize
      \lipschitzness \left( \lipschitzness + \vecnorm{\nabla f(x)}{2}
      \right) } \Big] & = \Prob_{z \sim \proposal_x} \Big[ \log Z(x) -
    \log Z(z) \geq -\stepsize \lipschitzness \left( \lipschitzness +
    \vecnorm{\nabla f(x)}{2} \right) \Big] \\
& \overset{(i)}{ \geq} \Prob_{z \sim \proposal_x} \Big[ G(z) - G(x -
    \stepsize \nabla f(x)) \geq - \stepsize \lipschitzness^2 \Big] \\
& \geq \frac{1}{2},
    \end{align*}
where step (i) follows from the inequality $G(x-\stepsize \nabla f(x)
)-G(x)\geq -\stepsize\lipschitzness \vecnorm{\nabla f(x)}{2}$, a
consequence of the $\lipschitzness$-Lipschitz nature of
$G$. Consequently, we can control the integral associated with $Q_1$.
    
Introducing the shorthand $A \mydefn \left\{ Q_1(x, z) \geq e^{
  -\stepsize \lipschitzness \left( \lipschitzness + \vecnorm{\nabla
    f(x)}{2} \right) } \right\}$, we have
\begin{align}
  \label{eq:q1}
\int p(x, z) \max(0, 1 - Q_1(x, z)) dz \overset{(i)}{\leq} & \Prob_{z
  \sim \proposal_x} \left( A\right) \left( 1 - e^{ -\stepsize
  \lipschitzness \left( \lipschitzness + \vecnorm{\nabla f(x)}{2}
  \right)} \right) + \Prob_{z \sim \proposal_x} \left( A^C
\right)\nonumber \\
\overset{(ii)}{\leq} &1 - \frac{1}{2} e^{- \stepsize \lipschitzness
  (\lipschitzness + \vecnorm{\nabla f(x)}{2})} \nonumber
\\
\overset{(iii)}{\leq} & \frac{1}{2} \left(1 + \stepsize \lipschitzness
(\lipschitzness + \vecnorm{\nabla f(x)}{2}) \right),
    \end{align}
where step (i) follows from decomposing the probability space into $A$
and $A^c$, with bounds on each event; step (ii) follows by the lower
bound on the probability of $A$ derived above, and step (iii) follows
as $1+x\leq \exp(x)$ for $x\in\real$.
    
Next, the function $Q_2$ can be controlled using the smoothness of
$f$:
\begin{align}
  \label{eq:logq2}
  \log Q_2(x, z) = & - (f(z) - f(x)) - \frac{1}{4 \stepsize}
  \vecnorm{x - z + \stepsize \nabla f(z)}{2}^2 + \frac{1}{4 \stepsize}
  \vecnorm{z - x + \stepsize \nabla f(x)}{2}^2 \nonumber \\
  = & \frac{1}{2} \left( f(x) - f(z) - \inprod{x - z}{\nabla f(x)}
  \right) + \frac{1}{2} \left( f(x) - f(z) - \inprod{x - z}{\nabla
    f(z)} \right) \nonumber\\ &+ \frac{\stepsize}{4} \left(
  \vecnorm{\nabla f(x)}{2} - \vecnorm{\nabla f(z)}{2} \right)
  \left( \vecnorm{\nabla f(x)}{2} + \vecnorm{\nabla f(z)}{2}
  \right) \nonumber \\
  \overset{(i)}{\geq} & - 2\smoothness \vecnorm{x - z}{2}^2 -
  \stepsize \vecnorm{\nabla f(x) - \nabla f(z)}{2} \left( 2
  \vecnorm{\nabla f(x)}{2} + \smoothness \vecnorm{x - z}{2} \right)
  \nonumber \\
  \overset{(ii)}{\geq} & - 3 \smoothness \vecnorm{x - z}{2}^2 - 2
  \stepsize \smoothness \vecnorm{x - z}{2} \cdot \vecnorm{\nabla
    f(x)}{2} \nonumber\\ \overset{(iii)}{\geq} & - 4 \smoothness
  \vecnorm{x - z}{2}^2 - \stepsize^2 \smoothness \vecnorm{\nabla
    f(x)}{2}^2,
\end{align}
where step (i) follows since $f(z) - f(x) - \inprod{z - x}{\nabla
  f(x)} \leq \frac{L}{2}\vecnorm{z-x}{2}$ and
$-\frac{L}{2}\vecnorm{z-x}{2} \leq f(x) - f(z) - \inprod{x - z}{\nabla
  f(z)}$ by smoothness of $f$; step (ii) follows from $\vecnorm{\nabla
  f(x) - \nabla f(z)}{2}\leq L \vecnorm{x-z}{2}$ by smoothness of $f$;
and step (iii) follows from Young's inequality.  Therefore, we obtain:
 \begin{align}
 \label{eq:q2}
 \int p(x, z) \max(0, 1 - Q_2(x, z)) dz \overset{(i)}{\leq} & \int
 p(x, z) \max(0, - \log Q_2(x, z)) dz
 \nonumber \\
 \overset{(ii)}{\leq}& \smoothness \int p(x, z) \left( 4
 \vecnorm{x - z}{2}^2 + \stepsize^2 \vecnorm{\nabla f(x)}{2}^2 \right)
 dz \nonumber \\
 \overset{(iii)}{\leq} & 37 \stepsize \smoothness d + 108 \stepsize^2
 \smoothness \left( \vecnorm{\nabla f(x)}{2}^2 + \lipschitzness^2
 \right),
 \end{align}
where step (i) follows from the elementary inequality $\log(x) \leq
x-1$ for $x>0$; step (ii) follows from equation~\eqref{eq:logq2}; and
step (iii) follows from Lemma~\ref{corr-coarse-control-proposal}.
Combining equations~\eqref{eq:dtvprop}, \eqref{eq:q1}
and~\eqref{eq:q2} yields the final conclusion.

    
\subsubsection{Auxiliary lemmas for the proof of Lemma~\ref{lemma-acc-rej}}
\label{app:proofcoupling}

In order to control the acceptance-rejection probability, we need the
following technical lemma:
\begin{lemma}
\label{lemma-expg-coupling}
Under Assumption~\ref{assume-convex-lip-regularizer}, given any fixed
$y \in \real^\dims$, there exists a coupling $\gamma$ such that for
$(X_1, X_2) \sim \gamma$, we have the marginals $X_1, X_2 \sim
\oracle_{\stepsize, g}(y) $, and
\begin{align*}
  \vecnorm{\frac{X_1 + X_2}{2} - y}{2} \leq \stepsize
  \lipschitzness,\quad a.s.
\end{align*}
\end{lemma}
\begin{proof}
We prove the existence of such a coupling by an explicit construction.
For any $y \in \real^\dims$ fixed, let the process $\xi_t$ and
$\zeta_t$ be defined as the solutions to the following SDEs, driven by
a Brownian motion $(B_t: t \geq 0)$.
\begin{align*}
  d \xi_t & = - \left(\frac{\xi_t - y}{2 \stepsize} + \nabla g(\xi_t)
  \right) dt + \sqrt{2} dB_t, \quad \xi_0 = y \\
  d \zeta_t & = - \left(\frac{\zeta_t - y}{2 \stepsize} + \nabla
  g(\zeta_t) \right) dt - \sqrt{2} dB_t, \quad \zeta_0 = y.
\end{align*}
Summing together the above equations yields
\begin{align*}
  d \left(\frac{\xi_t + \zeta_t}{2} - y \right) = - \frac{1}{2
    \stepsize} \left( \frac{\xi_t + \zeta_t}{2} - y\right) dt -
  \frac{1}{2} (\nabla g(\xi_t) + \nabla g(\zeta_t)) dt,
\end{align*}
which implies that $\left(\frac{\xi_t + \zeta_t}{2} - y \right)$
is a locally Lipschitz function of $t$. Consequently, we have
\begin{align*}
  \vecnorm{\frac{\xi_t + \zeta_t}{2} - y}{2}^2 = & - \frac{1}{
    \stepsize} \int_0^{+\infty} \vecnorm{ \frac{\xi_t + \zeta_t}{2} -
    y}{2}^2 dt - \int_0^{+\infty} \inprod{\nabla g(\xi_t) + \nabla
    g(\zeta_t)) }{\frac{\xi_t + \zeta_t}{2} - y } dt \\
  \leq & - \frac{1}{ \stepsize} \int_0^{+\infty} \vecnorm{
    \frac{\xi_t + \zeta_t}{2} - y}{2}^2 dt + \int_0^{+\infty}
  2\lipschitzness\vecnorm{\frac{\xi_t + \zeta_t}{2} - y }{2}
  dt \\
  \leq & \frac{1}{ 2 \stepsize} \int_0^{+\infty} \left(-
  \vecnorm{ \frac{\xi_t + \zeta_t}{2} - y}{2}^2 + \stepsize^2
  \lipschitzness^2 \right)dt.
\end{align*}
Now Gr\"{o}nwall's inequality guarantees that $\lim_{t\rightarrow +
  \infty} \vecnorm{\frac{\xi_t + \zeta_t}{2} - y}{2} \leq \stepsize
\lipschitzness$ almost surely, which completes the proof.
\end{proof}


\begin{corollary}
  \label{cor-proposal-bias}
  For any given $x \in \real^\dims$, we have
  \begin{align*}
    \vecnorm{\Exs_{Y\sim \proposal_x} Y - x}{2} \leq
    \stepsize(\lipschitzness + \vecnorm{\nabla f(x)}{2}).
  \end{align*}
\end{corollary}
\begin{proof}
Let $\tilde y = x - \stepsize \nabla f(x)$. By
Lemma~\ref{lemma-expg-coupling}, there exists a coupling $\gamma$ on
$\real^\dims \times \real^\dims$ such that for $(X_1, X_2) \sim
\gamma$, there is $X_1, X_2 \sim \proposal_x$ and $\vecnorm{\frac{X_1
    + X_2}{2} - \tilde y}{2} \leq \stepsize \lipschitzness$ almost
surely. We obtain:
\begin{align*}
  \vecnorm{\Exs_{Y\sim \proposal_x} Y - \tilde y}{2} =
  \vecnorm{ \frac{1}{2} \left(\Exs X_1 - \tilde y
    \right) + \frac{1}{2} \left(\Exs X_2 - \tilde y
    \right) }{2} \leq \Exs \vecnorm{\frac{1}{2}(X_1 +
    X_2) - \tilde y}{2} \leq \stepsize \lipschitzness.
\end{align*}
By the definition of $\tilde y$ we have $\vecnorm{x - \tilde y}{2}
\leq \stepsize \vecnorm{\nabla f(x)}{2}$, which concludes the proof.
\end{proof}

\begin{lemma}
\label{corr-coarse-control-proposal}
For any given $x \in \real^d$ and $Y \sim \proposal_x$, if
$\stepsize < \frac{1}{16( 1 + \smoothness)}$, there is:
\begin{align*}
  \Exs \vecnorm{ Y - x}{2}^2 \leq 12 \stepsize d + 36 \stepsize^2
  \left( \vecnorm{\nabla f(x)}{2}^2 + \lipschitzness^2 \right).
\end{align*}
\end{lemma}

\begin{proof}
  Note that:
\begin{align*}
  \inprod{- \nabla_y \log \proposal_x (y)}{y - x} = & \inprod{
    \frac{1}{2 \stepsize} (y - x + \stepsize \nabla f(x)) + \nabla
    g(y) }{y - x} \\
  \geq & \frac{1}{2 \stepsize} \vecnorm{y - x}{2}^2 - \vecnorm{y -
    x}{2} \left( \vecnorm{\nabla f(x)}{2} + \vecnorm{\nabla g(x)}{2}
  \right) \\
  \geq & \frac{1}{12 \stepsize} \vecnorm{y - x}{2}^2 - 3 \stepsize
  \left( \vecnorm{\nabla f(x)}{2} + \lipschitzness \right)^2.
\end{align*}
Applying Lemma~\ref{lemma-dissipative-tail} yields the claim.
\end{proof}


\subsection{Overlap bound for the transition kernels}
\label{app:overlap}

In this section, we prove the following proposition:
\begin{proposition}
  \label{corr-overlap-transition}
 There are universal constants $c, C$ such that for any convex set
 $\Omega \subseteq \real^d$ and stepsize $\stepsize \leq ( C
 (\smoothness d + \lipschitzness^2 + \sup_{x \in \Omega}
 \vecnorm{\nabla f(x)}{2}^2) )^{-1}$ and any pair $x_1, x_2 \in
 \Omega$ such that
  \begin{align*}
    \vecnorm{x_1 - x_2}{2} \leq c \min \left( \sqrt{\stepsize}, (
    \sup_{x \in \Omega}\vecnorm{\nabla f(x)}{2} + \lipschitzness +
    \smoothness \sqrt{\stepsize d} )^{-1} \right),
      \end{align*} 
we have $\DTV \left( \transition_{x_1}, \transition_{x_2} \right) \leq
\frac{9}{10}$.
\end{proposition}

We first provide a complete proof for
Proposition~\ref{corr-overlap-transition}, which is a straightforward
consequence of Lemma~\ref{lemma-overlap-succ-transition} and
Lemma~\ref{lemma:diff-rejection-prob}. Then we prove the two key
lemmas respectively. The auxiliary lemmas used in the proofs of
Lemma~\ref{lemma-overlap-succ-transition} and
Lemma~\ref{lemma:diff-rejection-prob} are proved in the last
subsection.


\subsubsection{Proof of Proposition~\ref{corr-overlap-transition}}

Let $\vecnorm{\cdot}{L^\infty (\real^d)^*}$ denotes the dual norm of
the $L^\infty (\real^d)$-norm, which is the generalization of TV to
arbitrary signed measures.  With this notation, we have
\begin{align*}
    \DTV \left( \transition_{x_1}, \transition_{x_2} \right) = & \DTV
    \left( \prej_{x_1} \delta_{x_1} + (1 - \prej_{x_1})
    \transition^{succ}_{x_1}, \prej_{x_2} \delta_{x_2} + (1 -
    \prej_{x_2}) \transition^{succ}_{x_2} \right)\\ \leq & \frac{1}{2}
    \vecnorm{\prej_{x_1} \delta_{x_1} - \prej_{x_2}
      \delta_{x_2}}{L^\infty (\real^d)^*} + \frac{1}{2} \vecnorm{ (1 -
      \prej_{x_1}) \transition^{succ}_{x_1} - (1 - \prej_{x_2})
      \transition^{succ}_{x_2} }{L^\infty (\real^d)^*}\\ \leq & \max
    (\prej_{x_1}, \prej_{x_2}) + |\prej_{x_1} - \prej_{x_2}| +
    \frac{1}{2} \vecnorm{ \transition^{succ}_{x_1} -
      \transition^{succ}_{x_2} }{L^\infty (\real^d)^*}\\ \leq & \max
    (\prej_{x_1}, \prej_{x_2}) + |\prej_{x_1} - \prej_{x_2}| + \DTV
    \left( \transition^{succ}_{x_1}, \transition^{succ}_{x_2} \right),
\end{align*}
Since the stepsize is upper bounded as $\stepsize \leq ( C
(\smoothness d + \lipschitzness^2 + \sup_{x \in \Omega}
\vecnorm{\nabla f(x)}{2}^2) )^{-1}$, Lemma~\ref{lemma-acc-rej} implies
that the first term is at most $\frac{7}{10}$.
On the other hand, suppose that
\begin{align*}
 \vecnorm{x_1 - x_2}{2} \leq c \min \left( \sqrt{\stepsize}, ( \sup_{x
   \in \Omega}\vecnorm{\nabla f(x)}{2} + \lipschitzness + \smoothness
 \sqrt{\stepsize d} )^{-1} \right).
\end{align*}
Then by applying Lemmas~\ref{lemma:diff-rejection-prob} and
Lemma~\ref{lemma-transition-succ-exp}, respectively, the second and
third terms are guaranteed to be bounded by $\frac{1}{10}$.  Combining
these three bounds, we find that
\begin{align*}
    \DTV \left( \transition_{x_1}, \transition_{x_2} \right) \leq
    \frac{7}{10} + \frac{1}{10} + \frac{1}{10} = \frac{9}{10},
\end{align*}
which finishes the proof.

\noindent Now we turn to the proofs of the two key lemmas.


\subsubsection{Proof of Lemma~\ref{lemma-overlap-succ-transition}}

Note that for any $x \in \real^\dims$, we can rewrite
\begin{align}
\label{eq:logsuc}
- \log \transition^{succ}_{x}(y) = \max \big \{ H_1 (x, y), H_2 (x, y)
\big \} + C(x),
\end{align}
where we define
\begin{align*}
H_1 (x, y) & \mydefn \frac{1}{4\stepsize} \vecnorm{y - x - \stepsize
  \nabla f(x)}{2}^2 + g(y) - G(x) \\
H_2(x, y) \mydefn &- f(x) + f(y) + g(y) + \frac{1}{4 \stepsize}
\vecnorm{y - x + \stepsize \nabla f(y)}{2}^2 - G(y), \quad \mbox{and}
\\
C(x) \mydefn & \log \int \min\left( p(x, z), e^{U(x) - U(z)} p(z, x)
\right) dz.
\end{align*}
For $x_1, x_2 \in \real^d$, by the symmetry of total variation
distance, we can assume $C(x_1) \leq C(x_2)$ without loss of
generality.  By Pinsker's inequality, we have
\begin{align*}
  \DTV \left( \transition^{succ}_{x_1}, \transition^{succ}_{x_2}
  \right) \leq \sqrt{\frac{1}{2} \kull{\transition^{succ}_{x_1}}{
      \transition^{succ}_{x_2}} }.
    \end{align*}
Comparing the function $H_1$ and $H_2$ at two different points, we
find that
\begin{align}
  &H_1 (x_1, y) - H_1 (x_2, y) \\ =& \frac{1}{4 \stepsize}
  \inprod{(x_1 - \stepsize \nabla f(x_1)) - (x_2 - \stepsize \nabla
    f(x_2))}{ (x_1 - \stepsize \nabla f(x_1)) + (x_2 - \stepsize
    \nabla f(x_2)) - 2y} \nonumber \\ &+ G(x_2) - G(x_1) \nonumber
  \\ & = \frac{1}{2\stepsize}
  \inprod{x_1-x_2}{x_2-y}+\frac{1}{4\stepsize} \inprod{x_1 - \stepsize
    \nabla f(x_1) - x_2 + \stepsize \nabla f(x_2)}{x_1 - \stepsize
    \nabla f(x_1) - x_2 - \stepsize \nabla f(x_2)} \nonumber \\ & +
  \frac{1}{2}\inprod{\nabla f(x_2)-\nabla f(x_1)}{x_2-y} + G(x_2) -
  G(x_1) ,\label{eq:h1}
\end{align}
\begin{align}
H_2 (x_1, y) - H_2 (x_2, y) = & \frac{1}{4 \stepsize} \inprod{x_1 -
  x_2}{x_1 + x_2 - 2y - 2 \stepsize \nabla f(y)} + f(x_2) -
f(x_1)\nonumber \\ &= \frac{1}{2\stepsize}
\inprod{x_1-x_2}{x_2-y}+\frac{1}{4\stepsize} \inprod{x_1 - x_2}{x_1 -
  x_2 - 2 \stepsize \nabla f(y)} + f(x_2) - f(x_1)
\label{eq:h2}.
\end{align}
So we have:
\begin{align*}
&\kull{\transition^{succ}_{x_2}}{\transition^{succ}_{x_1}}\\ = &
  \Exs_{Y \sim \transition^{succ}_{x_2}} \left( \max \left( H_1 (x_1,
  Y), H_2 (x_1, Y) \right) + C(x_1) - \max \left( H_1 (x_2, Y), H_2
  (x_2, Y) \right) - C(x_2) \right)\\ \overset{(i)}{\leq}& \Exs_{Y
    \sim \transition^{succ}_{x_2}} \left( \max \left( H_1 (x_1, Y),
  H_2 (x_1, Y) \right) - \max \left( H_1 (x_2, Y), H_2 (x_2, Y)
  \right) \right)\\ \overset{(ii)}{\leq}& \Exs_{Y \sim
    \transition^{succ}_{x_2} }\left(\max \left( H_1 (x_1, Y) -
  H_1(x_2, Y), H_2 (x_1, Y) - H_2 (x_2, Y)
  \right)\right)\\ \overset{(iii)}{\leq} & \underbrace{\frac{1}{2
      \stepsize} \Exs_{Y \sim \transition^{succ}_{x_2} } \inprod{x_1 -
      x_2}{x_2 - Y }}_{T_1}\\ &+ \underbrace{\frac{1}{4 \stepsize}
    \vecnorm{x_1 - \stepsize \nabla f(x_1) - x_2 + \stepsize \nabla
      f(x_2)}{2} \cdot \vecnorm{ x_1 - \stepsize \nabla f(x_1) - x_2 -
      \stepsize \nabla f(x_2)}{2}}_{T_2}\\ & + \underbrace{\frac{1}{4}
    \vecnorm{\nabla f(x_1) - \nabla f(x_2)}{2} \cdot \Exs_{Y \sim
      \transition^{succ}_{x_2}} \vecnorm{Y - x_2}{2} + |G(x_2) -
    G(x_1)|}_{T_3}\\ &+ \underbrace{\frac{1}{4 \stepsize}\vecnorm{x_1
      - x_2}{2} \Exs_{Y \sim \transition^{succ}_{x_2}} \vecnorm{x_1 -
      x_2 - 2 \stepsize \nabla f(Y)}{2} + |f(x_2) - f(x_1)|}_{T_4},
    \end{align*}
where step (i) follows as $C(x_1)\geq C(x_2)$; step (ii) follows from
the elementary inequality
\begin{align*}
  \max \{a,b \} - \max \{c,d \} \leq \max\{a-c,b-d\} \quad \mbox{$a,
    b, c, d \in \real^4$},
\end{align*}
and step (iii) follows from equations~\eqref{eq:h1} and~\eqref{eq:h2},
the elementary inequality $\max\{a+b,a+c\}\leq a+|b|+|c|$ for
$a,b,c\in\real^3$, combined with the Cauchy-Schwarz inequality.
    
Applying the Cauchy-Schwarz inequality and
Lemma~\ref{lemma-transition-succ-exp} to the term $T_1$, we obtain:
\begin{align*}
  T_1 = \frac{1}{2 \stepsize} \inprod{x_1 - x_2}{x_2 - \Exs_{Y \sim
      \transition^{succ}_{x_2} } (Y) } \leq 4 \vecnorm{x_1 - x_2}{2}
  \left( \vecnorm{\nabla f(x_2)}{2} + \lipschitzness + 2 \smoothness
  \sqrt{\stepsize d} \right).
\end{align*}
As for the second term $T_2$, we find
\begin{align*}
  T_2 = & \frac{1}{4 \stepsize} \vecnorm{x_1 - \stepsize \nabla f(x_1)
    - x_2 + \stepsize \nabla f(x_2)}{2} \cdot \vecnorm{ x_1 -
    \stepsize \nabla f(x_1) - x_2 - \stepsize \nabla
    f(x_2)}{2} \\
\overset{(i)}{\leq} & \frac{1}{4 \stepsize}(1+\stepsize
\smoothness)\vecnorm{x_1 - x_2}{2} \cdot \vecnorm{ x_1 -x_2- \stepsize
  \nabla f(x_1) - \stepsize \nabla f(x_2)}{2} \\ \leq&
\frac{1+\stepsize \smoothness}{4 \stepsize} \vecnorm{x_1 - x_2}{2}^2 +
\frac{1+\stepsize \smoothness}{4}\vecnorm{x_1 - x_2}{2} \left(
\vecnorm{\nabla f(x_1)}{2} + \vecnorm{\nabla f(x_2)}{2} \right),
\end{align*}
where step (i) follows from the smoothness of $f$.
   
For the last two terms, using
Lemma~\ref{corr-coarse-control-transition}, we obtain:
\begin{align*}
 T_3 = & \frac{1}{4} \vecnorm{\nabla f(x_1) - \nabla f(x_2)}{2} \cdot
 \Exs_{Y \sim \transition^{succ}_{x_2}} \vecnorm{Y - x_2}{2} + |G(x_2)
 - G(x_1)| \\ \leq & 6 \sqrt{\stepsize} \smoothness \vecnorm{x_1 -
   x_2}{2} \left(\sqrt{\stepsize} \vecnorm{\nabla f(x_2)}{2} +
 \sqrt{\stepsize}\lipschitzness + \sqrt{d} \right) + \lipschitzness
 \vecnorm{x_1 - x_2}{2}.\\ \\ T_4 =& \frac{1}{4 \stepsize}\vecnorm{x_1
   - x_2}{2} \Exs_{Y \sim \transition^{succ}_{x_2}} \vecnorm{x_1 - x_2
   - 2 \stepsize \nabla f(Y)}{2} + |f(x_2) - f(x_1)|\\ \leq&
 \frac{1}{2 \stepsize} \vecnorm{x_1 - x_2}{2}^2 + 12 \vecnorm{x_1 -
   x_2}{2} \cdot \left( 2(1+\stepsize \smoothness) \vecnorm{\nabla
   f(x_2)}{2} + \sqrt{\stepsize} \smoothness \left(
 \sqrt{\stepsize}\lipschitzness + \sqrt{d} \right) \right)\\ & +
 \vecnorm{\nabla f(x_2)}{2} \cdot \vecnorm{x_1 - x_2}{2} + \smoothness
 \vecnorm{x_1 - x_2}{2}^2.
    \end{align*}
Putting them together and using the assumption $\stepsize <
\frac{1}{16(\smoothness + 1)}$, we obtain:
\begin{align*}
  \kull{\transition^{succ}_{x_2}}{\transition^{succ}_{x_1}} \leq & T_1
  + T_2 + T_3 + T_4 \\ \leq & 20 \vecnorm{x_1 - x_2}{2} \left(
  \vecnorm{\nabla f(x_1)}{2} + \vecnorm{\nabla f(x_2)}{2} +
  \lipschitzness + L \sqrt{\stepsize d} \right) + 2 \left( \frac{1 }{
    \stepsize} + \smoothness \right) \vecnorm{x_1 - x_2}{2}^2.
\end{align*}
Substituting back into Pinsker's inequality completes the proof.


\subsubsection{Proof of Lemma~\ref{lemma:diff-rejection-prob}}

\noindent This section is devoted to the proof of the following lemma,
used in the proof of Proposition~\ref{corr-overlap-transition}.

\diffrejectionprob*

\begin{proof}
By definition, we have:
\begin{align*}
  \Prob_{Y \sim \transition_{x_1}} \left( Y = x_1\right) = \int \max
  \left( 0, p(x_1, y) - {e^{ U (x_1) - U (y)} p (y, x_1)} \right) dy.
\end{align*}
Adopting the shorthand $a \wedge b = \min(a,b)$ for $a,b \in \real$
and substituting into the quantity of interest yields
\begin{align*}
  &\Prob_{Y \sim \transition_{x_1}} \left( Y = x_1\right) - \Prob_{Y
    \sim \transition_{x_2}} \left( Y = x_2\right)\\ = & \int \max
  \left( 0, p(x_1, y) - {e^{ U (x_1) - U (y)} p (y, x_1)} \right) dy -
  \int \max \left( 0, p(x_2, y) - {e^{ U (x_2) - U (y)} p (y, x_2)}
  \right) dy\\ \overset{(i)}{\leq} & \underbrace{\int |p(x_1, y) -
    p(x_2, y)| dy}_{I_1} + \underbrace{\int \abss{ p (y, x_1) e^{ U
        (x_1) - U (y)}\wedge p (x_1, y) - p (y, x_2) e^{ U (x_2) - U
        (y)} \wedge p (x_2, y) } dy}_{I_2}\\ = & \; I_1 + I_2,
\end{align*}
where step (i) follows by combining that $\max(0,a-b)=\max(0,a-a\wedge
b)$ and $|\max (0,a-b)-\max(0,c-d))|\leq |a-c|+|b-d|$ to obtain
$\max(0,a-b)-\max(0,c-d)\leq |a-c |+|a\wedge b-c\wedge d) |$ for
$a,b,c,d\in\real$.

We now turn to controlling $I_2$. Define the function
$A_y (x) \mydefn  p(x, y) \wedge e^{ U (x) - U (y)} p (y,
  x)$, which is equal to
\begin{align*}
e^{- g(y)} \min \left( \frac{\exp\left( - \frac{\vecnorm{y - x +
      \stepsize \nabla f(x)}{2}^2 }{4 \stepsize} \right)}{Z(x)},
\frac{\exp \left( - \frac{\vecnorm{x - y +\stepsize \nabla f(y)}{2}^2
  }{4 \stepsize} + f(x) - f(y)\right)}{Z(y)} \right).
\end{align*}
A direct calculation yields
\begin{align*}
    A_y(x)^{-1} \nabla_x A_y (x)
    & = \left( \nabla G(x) - \frac{1}{2
      \stepsize} (I - \stepsize \nabla^2 f(x)) \left(x - \stepsize
    \nabla f(x) - y \right) \right) \bm{1}_{E(x, y)} \\*
    & \qquad {}
    + \left( \nabla
    f(x) - \frac{x - y + \stepsize \nabla f(y)}{2 \stepsize} \right)
    \bm{1}_{E(x, y)^C} \\
    & = \frac{y - x}{2 \stepsize} + \left( \nabla G(x) + \frac{1}{2}
    \nabla^2 f(x) \left(x - \stepsize \nabla f(x) - y \right) +
    \frac{1}{2} \nabla f(x) \right) \bm{1}_{E(x, y)} \\
    & + \left( \nabla f(x) - \frac{1}{2} \nabla f(y) \right)
    \bm{1}_{E(x, y)^C},
\end{align*}
where we define
\begin{align*}
  E \mydefn \left\{ Z(x)^{-1} \exp\left( - \frac{\vecnorm{y - x +
      \stepsize \nabla f(x)}{2}^2 }{4 \stepsize} \right) <
  Z(y)^{-1}\exp \left( - \frac{\vecnorm{y - x - \stepsize \nabla
      f(y)}{2}^2 }{4 \stepsize} + f(x) - f(y)\right) \right\}.
\end{align*}
Cosidering the line segment $z_\lambda \mydefn (1 - \lambda) x_1 +
\lambda x_2,~\lambda\in [0, 1]$, we have:
\begin{align*}
  I_2 & = \int_{\real^d} \abss{A_y (x_1) - A_y (x_2)} dy\\
& = \int_{\real^d} \abss{\int_0^1 \inprod{\nabla A_y (z_\lambda)}{x_1
      - x_2} d \lambda} dy \\
  & \leq \int_{\real^d} \int_0^1 \abss{\inprod{\nabla
      A_y(z_\lambda)}{x_1 - x_2}} d\lambda dy \\
  & \leq \int_0^1 \underbrace{\int_{\real^d} A_y (z_\lambda)
          \abss{\inprod{\frac{z_\lambda - y}{2 \stepsize}}{x_1 - x_2}}
          dy}_{T_1 (z_\lambda)} d \lambda \\
  &+\int_0^1 \underbrace{\int_{\real^d} A_y (z_\lambda)
    \vecnorm{x_1 - x_2}{2} \cdot \left( \vecnorm{\nabla
      G(z_\lambda)}{2} + \frac{1}{2} \opnorm{\nabla^2
      f(z_\lambda)} \vecnorm{z_\lambda - \stepsize \nabla
      f(z_\lambda) - y }{2} \right) dy}_{T_2 (z_\lambda)}
  d\lambda \\
        &+ \int_0^1 \underbrace{\int_{\real^d} A_y (z_\lambda)
    \vecnorm{x_1 - x_2}{2} \cdot \left( \frac{1}{2} \vecnorm{
      \nabla f(z_\lambda)}{2} + \vecnorm{\nabla f(z_\lambda) -
      \frac{1}{2}\nabla f(y)}{2} \right) dy }_{T_3 (z_\lambda)}
  d\lambda \\
& = \int_0^1 \left( T_1 (z_\lambda) + T_2 (z_\lambda) + T_3
        (z_\lambda) \right) d \lambda.
    \end{align*}
    
Now we turn to bounding the terms $T_1, T_2, T_3$. We use the fact
that $A_y (x) \leq p(x, y)$, so that the three integral terms can be
upper bounded by taking expectations under $\proposal$.

Beginning with the term $T_1$, note that:
\begin{align}
T_1(x) = & \int_{\real^d} A_y (x) \abss{\inprod{\frac{x - y}{2
      \stepsize}}{x_1 - x_2}} dy \nonumber\\ \leq &\frac{1}{2
  \stepsize} \Exs_{Y \sim \proposal_x} \abss{\inprod{x - Y}{x_1 -
    x_2}} \nonumber \\
\overset{(i)}{\leq}& \frac{1}{2 \stepsize} \abss{\inprod{x - \Exs_{Y
      \sim \proposal_x} Y}{x_1 - x_2}} + \frac{1}{2 \stepsize} \Exs_{Y
  \sim \proposal_x} \abss{\inprod{Y - \left(\Exs_{\xi \sim
      \proposal_x} \xi \right)}{x_1 -
    x_2}}\nonumber \\
\overset{(ii)}{\leq} & \frac{1}{2 \stepsize} \vecnorm{x - \Exs_{Y \sim
    \proposal_x} Y}{2} \cdot \vecnorm{x_1 - x_2}{2} + \frac{1}{2
  \stepsize} \sqrt{\Exs_{Y \sim \proposal_x} \left(\inprod{Y -
    \left(\Exs_{\xi \sim \proposal_x} \xi \right)}{x_1 - x_2}
  \right)^2} \nonumber \\
\overset{(iii)}{\leq} & \vecnorm{x_1 - x_2}{2} \left( \vecnorm{\nabla
  f(x)}{2} + \lipschitzness \right) + \frac{\vecnorm{x_1 -
    x_2}{2}}{\sqrt{2 \stepsize}}, \label{eq:tone}
 \end{align}
where step (i) follows by Minkowski's inequality on $\Exs |\cdot|$;
step (ii) follows by using Cauchy-Schwarz inequality on $\real^d$ for
the first term and on $L_2(\real)$ for the second term; and step (iii)
follows by applying Corollary~\ref{cor-proposal-bias} for the first
term and Lemma~\ref{lemma-var-one-direction} for the second term.

Turning to the term $T_2$, we have:
\begin{align}
  T_2 = &\int_{\real^d} A_y (x) \vecnorm{x_1 - x_2}{2} \cdot \left(
  \vecnorm{\nabla G(x)}{2} + \frac{1}{2} \opnorm{\nabla^2 f(x)}
  \vecnorm{x - \stepsize \nabla f(x) - y }{2} \right) dy \nonumber
  \\
  \overset{(i)}{\leq}& \vecnorm{x_1 - x_2}{2} \cdot \Exs_{Y \sim
    \proposal_x} \left( \vecnorm{\nabla G(x)}{2} + \frac{1}{2}
  \opnorm{\nabla^2 f(x)} \vecnorm{x - \stepsize \nabla f(x) - Y }{2}
  \right) \nonumber \\
  \overset{(ii)}{\leq} & \vecnorm{x_1 - x_2}{2} \cdot \left(
  \lipschitzness + \frac{\smoothness}{2} \left( \stepsize
  \vecnorm{\nabla f(x)}{2} + \sqrt{\Exs_{Y \sim \proposal_x}
    \vecnorm{Y - x}{2}^2} \right) \right)\nonumber
  \\
  \overset{(iii)}{\leq}& \vecnorm{x_1 - x_2}{2} \cdot \left(
  \lipschitzness + 3 \smoothness \left( \stepsize \vecnorm{\nabla
    f(x)}{2} + \stepsize \lipschitzness + \sqrt{\stepsize d} \right)
  \right), \label{eq:ttwo}
\end{align}
where step (i) follows since $A_y (x) \leq p(x, y)$; step (ii) follows
as $G$ is $\lipschitzness$-Lipschitz (see proof of
Lemma~\ref{lemma-acc-rej} in Section~\ref{app:lemma-acc-rej}), $f$ is
$\smoothness$-smooth and by Cauchy-Schwarz inequality on $L_2$; and
step (iii) follows from Lemma~\ref{corr-coarse-control-proposal}.
    
Turning to the term $T_3$, we have:
\begin{align}
  T_3 = & \int_{\real^d} A_y (x) \vecnorm{x_1 - x_2}{2} \cdot
  \left( \frac{1}{2} \vecnorm{ \nabla f(x)}{2} + \vecnorm{\nabla
    f(x) - \frac{1}{2}\nabla f(y)}{2} \right) dy \nonumber
  \\
      \overset{(i)}{\leq} & \vecnorm{x_1 - x_2}{2} \cdot \Exs_{Y \sim
        \proposal_x} \left( \frac{1}{2} \vecnorm{ \nabla f(x)}{2} +
      \vecnorm{\nabla f(x) - \frac{1}{2}\nabla f(Y)}{2} \right)
      \nonumber \\
        \overset{(ii)}{\leq} & \vecnorm{x_1 - x_2}{2} \left(
        \vecnorm{\nabla f(x)}{2} + \smoothness/2 \sqrt{\Exs_{Y \sim
            \proposal_x} \vecnorm{Y - x}{2}^2} \right)\nonumber \\
 \overset{(iii)}{\leq} & 3 \vecnorm{x_1 - x_2}{2} \left(
 \vecnorm{\nabla f(x)}{2} + \smoothness /2\left( \stepsize
 \vecnorm{\nabla f(x)}{2} + \stepsize \lipschitzness + \sqrt{\stepsize
   d} \right) \right). \label{eq:tthree}
\end{align}
where step (i) follows as $A_y (x) \leq p(x, y)$; step (ii) follows as
$f$ is $\smoothness$-smooth and by Cauchy-Schwarz inequality on $L_2$;
and step (iii) follows from Lemma~\ref{corr-coarse-control-proposal}.

Putting together equations~\eqref{eq:tone}, \eqref{eq:ttwo}
and~\eqref{eq:tthree}, and using the fact that $\stepsize < 1 / 16
\smoothness$ yields
\begin{align}
\label{eq:itwo}
I_2 \leq \frac{\vecnorm{x_1 - x_2}{2}}{\sqrt{2 \stepsize}} + C
\vecnorm{x_1 - x_2}{2} \left( \sup_{0 \leq \lambda \leq
  1}\vecnorm{\nabla f((1 - \lambda) x_1 + \lambda x_2)}{2} +
\lipschitzness + \smoothness \sqrt{\stepsize d} \right),
\end{align}
for universal constant $C > 0$.
    
The integral $I_1$ is relatively easy to control, since it is actually
a TV distance---viz.
\begin{align*}
  I_1 = & \DTV \left( \proposal_{x_1}, \proposal_{x_2} \right)
  \overset{(i)}{\leq} \sqrt{\frac{1}{2} D_{KL} \left( \proposal_{x_1}
    || \proposal_{x_2} \right) } \overset{(ii)}{\leq} \sqrt{\stepsize
    I \left( \proposal_{x_1} || \proposal_{x_2} \right) },
    \end{align*}
where step (i) follows from Pinsker's inequality, whereas step (ii) is
a consequence of the log-Sobolev inequality.  (Note that the density
of $\proposal_{x_2}$ is $\frac{1}{2 \stepsize}$-strongly log-concave.)
    
The Fisher information can be controlled as:
\begin{align*}
  I \left( \proposal_{x_1} || \proposal_{x_2} \right) = &
  \int_{\real^d} p(x_1, y) \vecnorm{\nabla_y \log p(x_1, y) - \nabla_y
    \log p(x_2, y)}{2}^2 dy\\ = & \int_{\real^d} p(x_1, y) \vecnorm{ -
    \frac{x_1 - \stepsize \nabla f(x_1) - y}{2 \stepsize} - \nabla
    g(y) + \frac{x_2 - \stepsize \nabla f(x_2) - y}{2 \stepsize} +
    \nabla g(y)}{2}^2 dy\\ \overset{(i)}{\leq} & \frac{1}{4
    \stepsize^2} (1 + \stepsize\smoothness)^2 \vecnorm{x_1 -
    x_2}{2}^2,
    \end{align*}
where step (i) follows since $f$ is $\smoothness$-smooth.  So for
$\stepsize < 1 / 16 \smoothness$, we have:
\begin{align}
  \label{eq:ione}
  I_1 \leq \frac{\vecnorm{x_1 - x_2}{2}}{\sqrt{ \stepsize}}.
\end{align}
Putting together equations~\eqref{eq:ione} and~\eqref{eq:itwo}
completes the proof.
\end{proof}


\subsubsection{Auxiliary lemmas for the proof of Lemma~\ref{lemma-overlap-succ-transition} and Lemma~\ref{lemma:diff-rejection-prob}}

\noindent This section is devoted to the proofs of some auxiliary
lemmas, which we state here.
\begin{lemma}
\label{lemma-transition-succ-exp}
 Under Assumptions~\ref{assume-smooth},~\ref{assume-dissipative}
 and~\ref{assume-convex-lip-regularizer}, for any given $x \in
 \real^d$ and $Y \sim \transition^{succ}_x$ and stepsize $\stepsize \in \big(0, 
 \frac{1}{16(\smoothness + 1)}\big)$, we have
   \begin{align*}
       \vecnorm{\Exs Y - x}{2} \leq 8\stepsize\left( \vecnorm{\nabla
         f(x)}{2} + \lipschitzness \right) + 16
       \stepsize^{\frac{3}{2}}\smoothness\sqrt{d}.
    \end{align*}
\end{lemma}
\begin{proof}
The argument is based on integration by parts: observing the density
of $\transition_x^{succ}$ is of the form $\exp \left( - \frac{1}{4
  \stepsize} \vecnorm{y - x}{2}^2 + \cdots\right)$, we pair $(y - x)$
with additional terms to make $\nabla_y \transition_x^{succ}(y)$
appear, which integrates to zero by Green's formula. Other terms
generated in this construction are accompanied with an $O(\stepsize)$
factor.

Concretely, noting that $\log \transition^{succ}_x (y)$ is almost
everywhere differentiable with respect to $y$, we can differentiate
equation~\eqref{eq:logsuc}. Doing so yields
    \begin{align}
       - \nabla_y \log \transition^{succ}_x (y) = & \left( \frac{1}{2
         \stepsize} (y - x - \stepsize \nabla f(x)) + \nabla g(y)
       \right) \bm{1}_{H_1(x, y) \geq H_2 (x, y)} \nonumber\\ & +
       \left( \frac{1}{2 \stepsize} (I_d + \stepsize \nabla^2 f(y)) (y
       - x + \stepsize \nabla f(y)) + \nabla f(y) + \nabla g(y) -
       \nabla G(y) \right) \bm{1}_{H_1(x, y) < H_2 (x, y)}
       \nonumber\\ = & \frac{1}{2 \stepsize} (y - x) + r_1 (x, y)
       \bm{1}_{H_1(x, y) \geq H_2 (x, y)} + r_2 (x, y) \bm{1}_{H_1(x,
         y) < H_2 (x, y)}, \label{eq:gradlogsuc}
    \end{align}
 where we define $r_1 \defn - \frac{1}{2} \nabla f(x) + \nabla g(y)$
 and $r_2 \defn \frac{1}{2} \nabla^2 f(y) (y - x) + \frac{1}{2}(3I_d +
 \stepsize \nabla^2 f(y)) \nabla f(y) + \nabla g(y) - \nabla G(y)$.
 Using Assumptions~\ref{assume-smooth}
 and~\ref{assume-convex-lip-regularizer}, we obtain:
\begin{align*}
  \vecnorm{r_1 (x, y)}{2} = & \vecnorm{- \frac{1}{2} \nabla f(x)
    + \nabla g(y) }{2} \leq \frac{1}{2}\vecnorm{\nabla f(x)}{2}
  + \lipschitzness.\\ \vecnorm{r_2 (x, y)}{2} = & \vecnorm{
    \frac{1}{2} \nabla^2 f(y) (y - x) + \frac{1}{2}(3I_d +
    \stepsize \nabla^2 f(y)) \nabla f(y) + \nabla g(y) - \nabla
    G(y) }{2}\\ \leq & \left(2 + \frac{\stepsize \smoothness}{2}
  \right) \smoothness \vecnorm{y - x}{2} + \frac{1}{2}(3 +
  \stepsize \smoothness) \vecnorm{\nabla f(x)}{2} + 2
  \lipschitzness.
\end{align*}
Note that $\min\left( p(x, y), e^{U(x) - U(y)} p(y, x) \right)$ is
almost everywhere differentiable with respect to $y$, and the
derivative is a pointwise function. Integrating yields:
\begin{align*}
  &\int (y - x) \transition^{succ}_x (y) dy\\ \overset{(i)}{=}
  &- 2 \stepsize \int \transition_x^{succ} (y) \nabla_y \log
  \transition_x^{succ} (y) dy - 2 \stepsize \int
  \transition_x^{succ} (y) \left( r_1 (x, y) \bm{1}_{H_1(x, y)
    \geq H_2 (x, y)} + r_2 (x, y) \bm{1}_{H_1(x, y) < H_2 (x, y)}
  \right) dy\\ \overset{(ii)}{=} &- 2 \stepsize \int
  \transition_x^{succ} (y) \left( r_1 (x, y) \bm{1}_{H_1(x, y)
    \geq H_2 (x, y)} + r_2 (x, y) \bm{1}_{H_1(x, y) < H_2 (x, y)}
  \right) dy,
\end{align*}
where step (i) follows from equation~\eqref{eq:gradlogsuc}, whereas
(ii) follows as $\int \transition_x^{succ} (y) \nabla_y \log
\transition_x^{succ} (y) dy=0$ by integration by parts.  Therefore, we
have:
\begin{align*}
  \vecnorm{\Exs Y - x}{2} =& 2 \stepsize \vecnorm{ \Exs \left( r_1 (x,
    Y) \bm{1}_{H_1(x, Y) \geq H_2 (x, Y)}\right) + \Exs\left( r_2 (x,
    Y) \bm{1}_{H_1(x, Y) < H_2 (x, Y) } \right) }{2}\\ \leq & 2
  \stepsize \left( \Exs \vecnorm{r_1 (x, Y)}{2} + \Exs \vecnorm{ r_2
    (x, Y)}{2}\right)\\ \leq & 2 \stepsize\left( \left(2 +
  \frac{\stepsize \smoothness}{2} \right)\left( \smoothness \vecnorm{Y
    - x}{2}+ \vecnorm{\nabla f(x)}{2}\right) + 3 \lipschitzness
  \right).
\end{align*}
By Lemma~\ref{corr-coarse-control-transition} and Cauchy-Schwartz
inequality, we obtain:
\begin{align*}
  \vecnorm{\Exs Y - x}{2} \leq 7\left(2 + \frac{\stepsize
    \smoothness}{2} \right) \stepsize^{3/2}L\sqrt{d} + \left(2 +
  \frac{\stepsize \smoothness}{2} \right)\left(1 + 6\stepsize
  \smoothness \right) \stepsize \vecnorm{\nabla f(x)}{2} +6 \left( 1 +
  2 \stepsize\smoothness \left(2 + \frac{\stepsize \smoothness}{2}
  \right) \right) \stepsize \lipschitzness .
\end{align*}
which leads to the final conclusion since $\stepsize < 1/ 16 \smoothness$.
\end{proof}


\begin{lemma}
  \label{corr-coarse-control-transition}
    For any given $x \in \real^d$ and $Y \sim \transition^{succ}_x$,
    if $\stepsize < \frac{1}{16( 1 + \smoothness)}$, there is:
    \begin{align*}
        \Exs \vecnorm{ Y - x}{2}^2 \leq 12 \stepsize d + 36
        \stepsize^2 \left( \vecnorm{\nabla f(x)}{2}^2 +
        \lipschitzness^2 \right).
    \end{align*}
\end{lemma}
\begin{proof}
    Note that:
    \begin{align*}
        \inprod{- \nabla_y \log \transition^{succ}_x (y)}{y - x} =&
        \inprod{\frac{1}{2 \stepsize} (y - x) + r_1 (x, y)
          \bm{1}_{H_1(x, y) \geq H_2 (x, y)} + r_2 (x, y)
          \bm{1}_{H_1(x, y) < H_2 (x, y)}}{y - x}\\ \geq & \frac{1}{2
          \stepsize}\vecnorm{y - x}{2}^2 - \vecnorm{y - x}{2} \left(
        \vecnorm{r_1(x, y)}{2} + \vecnorm{r_2(x, y)}{2} \right)\\ \geq
        & \frac{1}{2 \stepsize}\vecnorm{y - x}{2}^2 - 4 \smoothness
        \vecnorm{y - x}{2}^2 - (4 \vecnorm{\nabla f(x)}{2} +
        3\lipschitzness) \vecnorm{y - x}{2}\\ \geq & \left(\frac{1}{2
          \stepsize} - 4 \smoothness - \frac{1}{6 \stepsize}\right)
        \vecnorm{y - x}{2}^2 - 3 \stepsize (\vecnorm{\nabla f(x)}{2} +
        \lipschitzness)^2\\ \geq& \frac{1}{12 \stepsize} \vecnorm{y -
          x}{2}^2 - 3 \stepsize (\vecnorm{\nabla f(x)}{2} +
        \lipschitzness)^2.
    \end{align*}
    Using Lemma~\ref{lemma-dissipative-tail} we arrive at the result.
\end{proof}
\begin{lemma}\label{lemma-var-one-direction}
    For any $x \in \real^d$ and $Y \sim \proposal_x$, for any $v \in
    \sphere^{d - 1}$, we have:
    \begin{align*}
        \Exs \left( \inprod{v}{Y - (\Exs Y)} \right)^2 \leq 2 \stepsize.
    \end{align*}
\end{lemma}
\begin{proof}
 The proposal distribution $\proposal_x$ has a density proportional to
 $\exp \left( - \frac{\vecnorm{z - x + \stepsize \nabla f(x)}{2}^2 }{4
   \stepsize} - g(z) \right)$, which is $\frac{1}{2
   \stepsize}$-strongly log concave. Consequently, Harg\'{e}'s
 inequality~\citep{harge2004convex} guarantees that for any fixed
 convex function $\psi$ and fixed vector $v$, we have
    \begin{align*}
        \Exs \psi (v^T (Y - \Exs Y)) \leq \Exs \psi (v^T(\xi - \Exs
        \xi)),
    \end{align*}
    where $\xi \sim \mathcal{N} (0, 2 \stepsize I_d)$.  The claim
    follows by applying this inequality with the function $\psi(a) =
    a^2$.
\end{proof}


\section{Conclusion}

We have presented a new Metropolis-Hastings based algorithm to
efficiently sample from non-smooth composite potentials.  Our
algorithm is based on a new form of proposal distribution, one that is
inspired by the proximity operator defined by Moreau-Yoshida
regularization.  Under some mild regularity conditions, we prove that
the resulting algorithm has mixing scaling as $O(d\log(d/\varepsilon
))$, where $d$ denotes the dimension and $\varepsilon$ denotes the
desired tolerance in total variation norm.  This guarantee matches
known results for smooth potentials satisfying the same regularity
conditions, up to a multiple of the condition number.

Our work leaves open a number of directions worth pursuing in future
work.  Our work assumes that the regularizer in the composite
potential is Lipschitz; analyzing the more general case of
non-Lipschitz but convex regularizers, such as those that arise in
sampling with constraints, would be useful.  In addition, we have
analyzed a first-order sampling method, so that developing and
analyzing a higher-order sampling method, such as one based on the
Hamiltonian point of view, is a promising direction for further
research.

\section*{Acknowledgements}  
 This work was partially supported by Office of Naval Research Grant
 ONR-N00014-18-1-2640 to MJW and National Science Foundation Grant
 NSF-CCF-1909365 to PLB and MJW.  We also acknowledge support from
 National Science Foundation grant NSF-IIS-1619362 to PLB.

\bibliographystyle{abbrvnat}


\bibliography{reference}


\appendix


\section{Tail bounds for the process}
\label{app:tail}

Throughout the previous proofs, we have bounded the tails of the
process defined by the Metropolis-Hastings algorithm using various
auxiliary results, which are collected and proved here.

\begin{lemma}
\label{lemma-dissipative-tail}
Let $U: \real^d \rightarrow \real$ be an almost everywhere
differentiable function, satisfying the distant dissipativity
condition $\langle x, \nabla U (x) \rangle \geq a \Vert x \Vert^2 - b$
for all $x \in \Rspace^d$. Then there is a numerical constant $C > 0$
such that for all $\delta \in (0,1)$, we have
\begin{align}
\label{EqnDissipativeTail}  
    \mathbb{P}_{\pi} \left( \Vert X \Vert \geq C \sqrt{ \frac{ b + d +
        \log \frac{1}{ \delta}}{a}} \right) \leq \delta
  \end{align}
where $\mathbb{P}_\pi$ denotes the distribution with density function
$\pi \propto e^{- U}$.
\end{lemma}
\begin{proof}
Consider the Langevin diffusion defined by the It\^{o} SDE:
\begin{align}
\label{eq:langevin-diffusion}  
d X_t = - \nabla U ( X_t) dt + \sqrt{2} dB_t \quad \mbox{with initial
  condition $X_0 = 0$.}
  \end{align}
It is known (e.g.~\cite{markowich1999trend}) that under the
dissipativity condition given in the lemma statement, the Langevin
diffusion~\eqref{eq:langevin-diffusion} converges in $L^2$ to $\pi$.

In order to prove the claimed tail bound~\eqref{EqnDissipativeTail},
our strategy is fix a time $T > 0$, and obtain bounds on the moments
$\mathbb{E} \Vert X_T \Vert^p$ for all $p \geq 1$.  By taking limits
as $T$ goes to infinity, we then recover tail bounds for $X \sim \pi$.

\newcommand{\expconstantbdg}{\nu}

Invoking It\^{o}'s formula, for any $\expconstantbdg > 0$, we find
that
\begin{multline*}
\frac{1}{2} e^{ \expconstantbdg t} \Vert X_t \Vert^2 - \frac{1}{2} \Vert X_0 \Vert^2
= \int_0^t \langle X_s, - \nabla U ( X_s) e^{ \expconstantbdg s} \rangle ds +
\frac{d}{2} \int_0^t e^{\expconstantbdg s} ds + \int_0^t e^{ \expconstantbdg s} X_s^T dB_s \\ +
\frac{1}{2} \int_0^t \expconstantbdg e^{\expconstantbdg s} \Vert X_s \Vert^2 ds.
\end{multline*}
Let $M_t \mydefn \int_0^t X_s^T e^{ \expconstantbdg s} dB_s$ be the
martingale term.  Without loss of generality, we can assume that $p
\geq 4$.  (The claim for $p \in [1, 4]$ can be obtained from its
analogue for $p \geq 4$ by applying H\"{o}lder's inequality.)
Applying the Burkholder-Gundy-Davis
inequality~(\cite{revuz2013continuous}, Chapter 4.4) yields
\begin{align*}
  \mathbb{E} \sup_{0 \leq t \leq T} |M_t|^{ \frac{p}{2}} \leq (p C)^{
    \frac{p}{4}} \mathbb{E} \langle M, M \rangle_T^{ \frac{p}{4}} = &
  (p C)^{ \frac{p}{4}} \mathbb{E} \left( \int_0^T e^{ 2c s} X_s^T X_s
  ds \right)^{ \frac{p}{4}} \\ \leq & (p C)^{ \frac{p}{4}} \mathbb{E}
  \left( \int_0^T e^{ 2 \expconstantbdg s} \Vert X_s \Vert^2 ds \right)^{
    \frac{p}{4}} \\
  \leq & (p C)^{ \frac{p}{4}} \mathbb{E} \left( \sup_{0 \leq s \leq T}
  e^{ \expconstantbdg s} \Vert X_s \Vert^2 \cdot \int_0^T e^{ \expconstantbdg s} ds \right)^{
    \frac{p}{4}} \\
  \leq & \left( \frac{ C p e^{\expconstantbdg T}}{\expconstantbdg} \right)^{ \frac{p}{4}} \left(A
  + \frac{1}{A} \mathbb{E} \left( \sup_{0 \leq t \leq T} e^{ ct} \Vert
  X_t \Vert^2 \right)^{ \frac{p}{2}} \right),
\end{align*}
for an arbitrary $A$ which will be determined later. On the other
hand, by the assumption in the lemma, we have
\begin{equation*}
  \int_0^t \langle X_s, -\nabla U ( X_s) e^{ \expconstantbdg s} \rangle ds
  \leq \int_0^t \left( - a \Vert X _s \Vert^2 + b \right) e^{ \expconstantbdg
    s} ds.
\end{equation*}
Putting the above results together and letting $\expconstantbdg = 2 a$, we obtain
that
\begin{align*}
  \mathbb{E} \left( \sup_{0 \leq t \leq T} e^{2 a t} \Vert  X_t
  \Vert^2 \right)^{ \frac{p}{2}}
  & \leq 3^{ \frac{p}{2} - 1}
  \mathbb{E} \biggr( \sup_{0 \leq t \leq T}
  \int_0^t \biggr( 2 \langle X _s, - \nabla U( X _s)
  \rangle + d + \expconstantbdg \Vert  X_s \Vert^2 \biggr) 
  e^{cs} ds \biggr)^{ \frac{p}{2}} \\
  & \hspace{ 12 em} + 3^{ \frac{p}{2} - 1} \mathbb{E} \sup_{0 \leq t \leq T}
  |M_t|^{ \frac{p}{2}} \\
  & \leq \left( \frac{C p e^{2 a T}}{a} \right)^{ \frac{p}{4}}
  \left( A + \frac{1}{A} \mathbb{E} \left( \sup_{0 \leq t \leq 
    T} e^{2 a t} \Vert  X_t \Vert^2 \right)^{ \frac{p}{2}} \right) \\
  & \hspace{ 12 em} + 3^{ \frac{p}{2} - 1} \mathbb{E} \left( \sup_{0 \leq t \leq 
    T} \int_0^t \left(2 b + d \right)
  e^{2 a s} ds \right)^{ \frac{p}{2}},
\end{align*}
for some universal constant $C>0$.

By choosing $A = 2 \left( \frac{C p e^{ 2 a T}}{a} \right)^{
  \frac{p}{4}}$ and plugging that value into above inequality, we
achieve that
\begin{equation}
  \left( \mathbb{E} \Vert X_T \Vert^p \right)^{ \frac{1}{p}}
  \leq e^{- a T} \left( \mathbb{E} \left( \sup_{0 \leq t \leq 
    T} e^{ 2 a t} \Vert  X_t \Vert^2 \right)^{ \frac{p}{2}} \right)^{ \frac{1}{p}}
  \leq C' \left( \sqrt{ \frac{p}{a}} + \sqrt{ \frac{2 b + d}{a}} \right),
\end{equation}
for a universal constant $C'>0$. Letting $T \to +\infty$ leads to the
following inequality
\begin{align*}
  \mathbb{E}_\pi (\Vert X \Vert^p)^{ \frac{1}{p}}
  \lesssim \sqrt{ \frac{p + b + d}{a}}.
\end{align*}
Furthermore, for any $t>0$, we have
\begin{align*}
  \mathbb{P} \left(\Vert X \Vert \geq t \right)
  \leq \inf_{p \geq 1} (C')^p \frac{ \mathbb{E} \Vert X \Vert^p}{ t^p} 
  \leq \inf_{p \geq 1} (2C')^p \left( \left( \frac{p}{ a t^2} \right)^{ \frac{p}{2}}
  + \left( \frac{b + d}{ a t^2} \right)^{ \frac{p}{2}} \right).
\end{align*}
Given $\delta > 0$, by choosing $p = 2 \log \frac{2}{\delta}$ and
$t = 2 C'( \sqrt{ \frac{p}{a}} + \sqrt{\frac{b + d}{a}})$, we obtain the following inequality
\begin{align*}
  \mathbb{P} \left( \Vert X \Vert \geq t \right)
  \leq \left( \frac{ 4 C'^2 p}{a t^2} \right)^{ \frac{p}{2}} 
  + \left( \frac{ 4 C'^2 (b + d)}{a t^2} \right)^{ \frac{p}{2}}
  \leq \frac{\delta}{2} + \frac{ \delta}{2} 
  = \delta,
\end{align*}
which completes the proof.
\end{proof}


\noindent Now using Lemma~\ref{lemma-dissipative-tail} and
Assumption~\ref{assume-dissipative}, we can control the tail of the
target distribution:
\begin{lemma}
  \label{cor-tail-target}
  Under Assumption~\ref{assume-dissipative}
  and~\ref{assume-convex-lip-regularizer}, for any $s \in (0,1)$ and
  the radius
  \begin{align*}
    R_s \mydefn C\sqrt{\frac{\distantdissipative + d +
        \lipschitzness^2 + \log (1/s)}{\dissipative}} \quad \mbox{for
      a universal constant $C$},
  \end{align*}
we have $\pi \left( \ball (\referencepoint, R_s)\right) \geq 1 - s$.
\end{lemma}


\end{document}

%% file: final_macros.tex
\newcommand{\dims}{\ensuremath{d}}

\newcommand{\real}{\ensuremath{\mathbb{R}}}

\newcommand{\brackets}[1]{\left[ #1 \right]}

\newcommand{\abss}[1]{\left| #1 \right |}

\newcommand{\Rspace}{\ensuremath{\mathbb{R}}}

\newcommand{\ball}{\ensuremath{\mathbb{B}}}
\newcommand{\sphere}{\ensuremath{\mathbb{S}}}

\newcommand{\mydefn}{\ensuremath{:=}}

\newcommand{\smoothness}{L} 

\newcommand{\defn}{:=}

\newcommand{\matsnorm}[2]{|\!|\!| #1 | \! | \!|_{{#2}}}
\newcommand{\vecnorm}[2]{\left\| #1\right\|_{#2}}

\newcommand{\opnorm}[1]{\ensuremath{\matsnorm{#1}{\tiny{\mbox{op}}}}}

\newcommand{\inprod}[2]{\ensuremath{\langle #1 , \, #2 \rangle}}
\newcommand{\prox}[2]{\ensuremath{Prox_{ #1}( #2)}}

\newcommand{\kull}[2]{\ensuremath{D_{\text{KL}}(#1 \| #2)}}

\newcommand{\Exs}{\ensuremath{{\mathbb{E}}}}
\newcommand{\Prob}{\ensuremath{{\mathbb{P}}}}

\newcommand{\prej}{p^{rej}}

\newtheoremstyle{named}{}{}{\itshape}{}{\bfseries}{.}{.5em}{\thmnote{#3's }#1}
\theoremstyle{named}

\theoremstyle{plain}

\newtheorem{theorem}{Theorem}
\newtheorem{proposition}{Proposition}
\newtheorem{lemma}{Lemma}

\newtheorem{corollary}{Corollary}
\newtheorem{definition}{Definition}

\newlength{\widebarargwidth}
\newlength{\widebarargheight}
\newlength{\widebarargdepth}

\makeatletter
\long\def\@makecaption#1#2{
        \vskip 0.8ex
        \setbox\@tempboxa\hbox{\small {\bf #1:} #2}
        \parindent 1.5em  
        \dimen0=\hsize
        \advance\dimen0 by -3em
        \ifdim \wd\@tempboxa >\dimen0
                \hbox to \hsize{
                        \parindent 0em
                        \hfil
                        \parbox{\dimen0}{\def\baselinestretch{0.96}\small
                                {\bf #1.} #2
                                }
                        \hfil}
        \else \hbox to \hsize{\hfil \box\@tempboxa \hfil}
        \fi
        }
\makeatother

\long\def\comment#1{}
\definecolor{battleshipgrey}{rgb}{0.52, 0.52, 0.51}
\definecolor{darkgray}{rgb}{0.66, 0.66, 0.66}
\definecolor{darkgreen}{rgb}{0.0, 0.2, 0.13}
\definecolor{darkspringgreen}{rgb}{0.09, 0.45, 0.27}
\definecolor{dukeblue}{rgb}{0.0, 0.0, 0.61}
\definecolor{olivedrab7}{rgb}{0.24, 0.2, 0.12}
\definecolor{darkblue}{rgb}{0.0, 0.0, 0.55}
\definecolor{darkscarlet}{rgb}{0.34, 0.01, 0.1}
\definecolor{candyapplered}{rgb}{1.0, 0.03, 0.0}
\definecolor{ao(english)}{rgb}{0.0, 0.5, 0.0}
\definecolor{applegreen}{rgb}{0.55, 0.71, 0.0}

\newcommand{\red}[1]{\textcolor{red}{#1}}
\newcommand{\mjwcomment}[1]{{\bf{{\red{{MJW --- #1}}}}}}
